\documentclass{article}

\usepackage{microtype}
\usepackage{graphicx}
\usepackage{booktabs} 

\usepackage{hyperref}

\usepackage{microtype}
\usepackage{graphicx}
\usepackage{booktabs} 
\usepackage{graphicx} 
\usepackage{todonotes}
\usepackage{xcolor}
\usepackage{dsfont}
\usepackage{nicefrac}
\usepackage{tcolorbox}
\usepackage{tabularx}
\usepackage{pifont}
\usepackage{enumitem}
\usepackage{multirow}
\usepackage{caption}
\usepackage{wrapfig}
\usepackage[export]{adjustbox}
\usepackage[list=true]{subcaption}
\usepackage{tikz}
\usepackage{xfrac}
\usetikzlibrary{shapes,backgrounds}
\usetikzlibrary{arrows.meta, bending, positioning}

\usepackage[accepted]{icml2024}

\usepackage{amsmath}
\usepackage{amssymb}
\usepackage{mathtools}
\usepackage{amsthm}
\usepackage{enumitem}
\newcommand{\mypar}[1]{\textbf{#1.}}

\usepackage[capitalize,noabbrev]{cleveref}

\usepackage{amsthm}
\usepackage{thmtools, thm-restate}
\usepackage{algorithm, algorithmic}

\declaretheorem[]{challenged assumption}

\usepackage{xspace}
\DeclareRobustCommand{\eg}{e.g.,\@\xspace}
\DeclareRobustCommand{\ie}{i.e.,\@\xspace}
\DeclareRobustCommand{\wrt}{w.r.t.\@\xspace}
\definecolor{myviolet}{rgb}{0.6, 0.4, 0.8}
\definecolor{green}{rgb}{255, 0, 0}

\usepackage{amsfonts}[mathscr]
\usepackage{amssymb}
\usepackage{amsmath}
\newcommand{\Aspace}{\mathcal{A}}
\newcommand{\Sspace}{\mathcal{S}}

\newcommand{\ind}{\mathbb{I}}

\DeclareMathOperator{\PR}{\mathbb{P}}

\newcommand{\cmp}{\mathcal{M}}

\DeclareMathOperator*{\argmin}{arg\,min}
\newcommand{\regret}{\mathcal{R}}
\newcommand{\group}{\mathbb{G}}

\newcommand{\AspaceAbs}{\overline{\mathcal{A}}}
\newcommand{\SspaceAbs}{\overline{\mathcal{S}}}
\newcommand{\aAbs}{\overline{a}}
\newcommand{\sAbs}{\overline{s}}

\newcommand{\lambdaAbs}{\overline{\lambda}}
\newcommand{\Fabs}{\overline{\mathcal{F}}}
\newcommand{\Labs}{\overline{\mathcal{L}}}

\newcommand{\muAbs}{\overline{\mu}}

\newcommand{\AlgNameLong}{Geometric Active Exploration\xspace}
\newcommand{\AlgNameShort}{\textsc{\small{GAE}}\xspace}
\newcommand{\AlgNameDef}{\textbf{G}eometric \textbf{A}ctive \textbf{E}xploration (\textsc{\small{GAE}}\xspace)}
\newcommand{\mdpSolver}{\textsc{\small{MDP-SOLVER}}\xspace}

\newcounter{relctr} 
\newcommand\labelrel[2]{%
  \begingroup
    \refstepcounter{relctr}%
    \stackrel{\textnormal{(\arabic{relctr})}}{\mathstrut{#1}}%
    \originallabel{#2}%
  \endgroup
}
\AtBeginDocument{\let\originallabel\label} 

\allowdisplaybreaks

\usepackage{longtable}
\usepackage[title]{appendix}

\icmltitlerunning{Geometric Active Exploration in Markov Decision Processes}

\begin{document}

\twocolumn[
\icmltitle{Geometric Active Exploration in Markov Decision\\ Processes: the Benefit of Abstraction}



\icmlsetsymbol{equal}{*}

\begin{icmlauthorlist}
\icmlauthor{Riccardo De Santi}{eth,ethai}
\icmlauthor{Federico Arangath Joseph}{equal,eth}
\icmlauthor{Noah Liniger}{equal,eth}
\icmlauthor{Mirco Mutti}{tech}
\icmlauthor{Andreas Krause}{eth}
\end{icmlauthorlist}

\icmlaffiliation{eth}{Department of Computer Science, ETH Zurich, Zurich, Switzerland}
\icmlaffiliation{ethai}{ETH AI Center, Zurich, Switzerland}
\icmlaffiliation{tech}{Technion, Haifa, Israel}

\icmlcorrespondingauthor{Riccardo De Santi}{rdesanti@ethz.ch}

\icmlkeywords{MDP homomorphism, Active Exploration, Reinforcement learning, Machine Learning, ICML}

\vskip 0.3in
]



\printAffiliationsAndNotice{\icmlEqualContribution} 

\begin{abstract}
\label{sec:abstract}
\looseness-1 How can a scientist use a Reinforcement Learning (RL) algorithm to design experiments over a dynamical system's state space? In the case of finite and Markovian systems, an area called \emph{Active Exploration} (AE) relaxes the optimization problem of experiments design into Convex RL, a  generalization of RL admitting a wider notion of reward. Unfortunately, this framework is currently not scalable and the potential of AE is hindered by the vastness of experiment spaces typical of scientific discovery applications. However, these spaces are often endowed with natural geometries, \eg permutation invariance in molecular design, that an agent could leverage to improve the statistical and computational efficiency of AE. To achieve this, we bridge AE and MDP homomorphisms, which offer a way to exploit known geometric structures via abstraction. Towards this goal, we make two fundamental contributions: we extend MDP homomorphisms formalism to Convex RL, and we present, to the best of our knowledge, the first analysis that formally captures the benefit of abstraction via homomorphisms on sample efficiency. Ultimately, we propose the \AlgNameLong (\AlgNameShort) algorithm, which we analyse theoretically and experimentally in environments motivated by problems in scientific discovery.
\end{abstract}

\section{Introduction}
\label{sec:introduction}
\looseness -1 The problem of optimal experimental design (OED) \cite{chaloner1995bayesian} refers to the task of optimally selecting experiments to minimize a measure of uncertainty of an unknown \emph{quantity of interest} $f: \Sspace \to \mathbb{R}$, where $\Sspace$ denotes a space of experiments.  Typically, the problem considers a limited budget of resources, \eg number of experiments, and assumes the possibility to directly sample $f$ at arbitrary inputs $s \in \Sspace$. Conceptually, an optimal design can be interpreted as a distribution over experiments determining the probability with which these should be carried out in order to minimize the uncertainty of $f$~\cite{pukelsheim2006optimal}. 

\looseness -1 Interestingly, in a wide variety of applications the input space $\Sspace$ corresponds to the state space of a dynamical system ~\cite{mutny2023active}. Therefore, the agent carrying out the experiments must respect the underlying dynamics and cannot teleport from any experiment, now interpreted as a state $s \in \Sspace$, to any other experiment. For instance, consider the environmental sensing problem illustrated in Figure~\ref{fig:diffusion_drawing}, where an agent aims to actively estimate the amount of air pollution caused by the diffusion of a chemical substance released from a point source. To address this problem, the agent chooses sampling policies to minimize an estimation error of the amount of pollutant $f$. 

In the case of time-discrete and Markovian dynamical systems, this problem is known as Active Exploration (AE)~\cite{mutny2023active, tarbouriech2019active, tarbouriech2020active}. AE frames the experiments design task as an instance of Convex Reinforcement Learning (Convex RL)~\cite{hazan2019maxent, zahavy2021reward}, a recent generalization of RL where the agent aims to minimize a convex functional of the state-action distribution induced by a policy interacting with the environment. 

\looseness -1 The AE formulation of the OED problem on dynamical systems is promising as it allows to learn (from data) optimal sampling policies that respect the system dynamics while minimizing a measure of uncertainty of $f$. Nonetheless, solving an instance of Convex RL typically entails solving a sequence of Markov Decision Processes (MDPs) and estimating the visitation density at each iteration~\cite{hazan2019maxent}. As a consequence, current Active Exploration methods are not scalable, hindering their use in real-world scientific discovery problems where experiment spaces are generally immense~\cite{wang2023scientific, Thiede_2022}.

\begin{figure*}[t!]
\centering
{%
    \includegraphics[width=0.65 \textwidth]{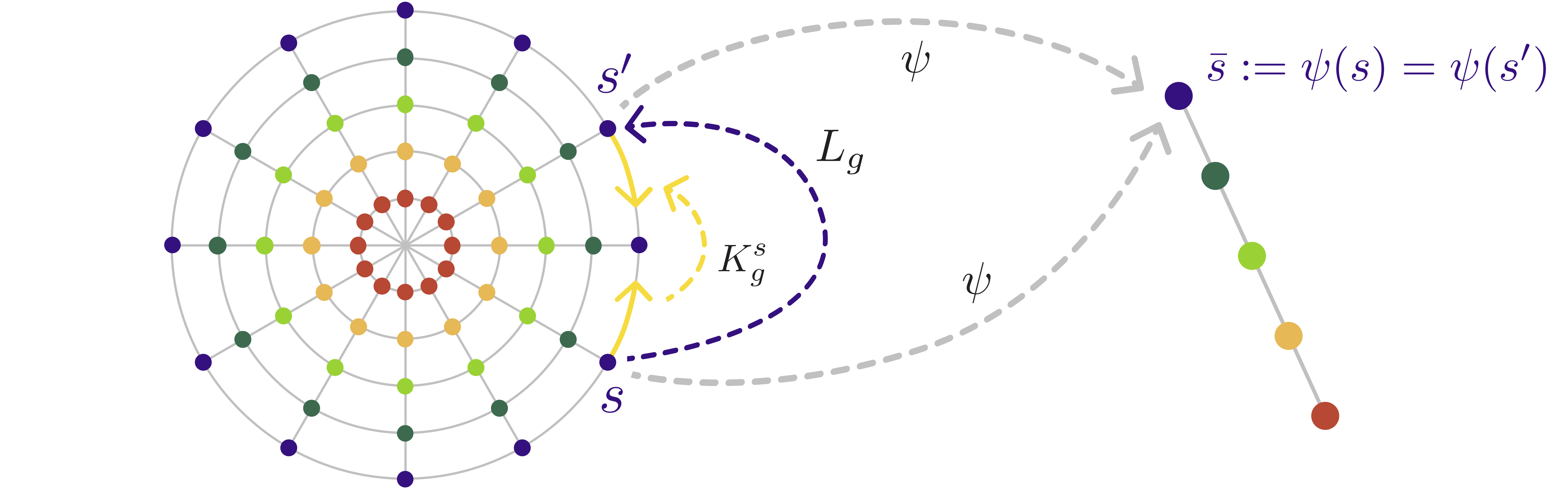}%
}%
\caption[Short Caption]{ Radial diffusion process of a pollutant from a central point source. On the left, original MDP where each circle is an $f$-equivalence class, $L_g$ denotes a state symmetry acting on $f$, $K_g^s$ denotes a state-dependent action symmetry acting on $P$. On the right, the abstract MDP obtained via the MDP homomorphism  $h = (\psi, \{\phi_s \mid s \in \Sspace\})$, where $\psi$ maps $f$-equivalence classes to abstract states.}
\label{fig:diffusion_drawing}
\end{figure*}
Luckily, these spaces are often endowed with natural geometries, as in the case of permutation invariances in molecular design~\cite{cheng2021molecular, elton2019deep}, or roto-translation and reflection invariances in environmental sensing~\cite{krause2008optimizing, derPol2020homomorphic}. Let us take the example reported in Figure \ref{fig:diffusion_drawing}. One can expect locations that are equally distant from the point source (center) to have almost identical amounts of pollutant \ie the quantity of interest $f$ follows radial symmetries.\footnote{Note that in the paper we conceptualize symmetries to be \emph{exact} while they may be \emph{approximate} in practice. Extending this work to deal with approximate symmetries is a nice direction for future works.} Therefore, by sampling $f$ in one location, the agent gains information on $f$ in all the other symmetric locations as well.
As a consequence, in this work we first aim to answer the following question:
\begin{center}
    \emph{How can a RL agent exploit geometric structure to increase the statistical and computational efficiency of AE?}
\end{center}
First, we introduce a novel geometric estimation error and corresponding AE objective (Sec. \ref{sec:problem} and \ref{sec:CRL_objective}). Then, we find optimal sampling strategies for the introduced AE objective by bridging Active Exploration with the area of MDP homomorphisms~\cite{ravindran2001symmetries, derPol2020homomorphic}, which offers an algorithmic scheme to leverage known geometric structure in RL via abstraction. Unfortunately, MDP homomorphisms are not directly usable in AE as it is not a classic RL problem. Thus, we extend MDP homomorphisms to Convex RL and introduce \AlgNameDef, an algorithm that solves the AE problem by exploiting known geometric structure via abstraction (Sec. \ref{sec:algorithm}). To the best of our knowledge, we provide the first analysis that formally captures the benefit of abstraction on sample efficiency via MDP homomorphisms (Sec. \ref{sec:analysis}). Finally, we showcase experimentally the statistical and computational advantages of \AlgNameShort in illustrative environments inspired by scientific discovery problems (Sec.~\ref{sec:experiments}). 

To sum up, we make the following contributions:
\begin{itemize}[noitemsep,topsep=0pt,parsep=0pt,partopsep=0pt,leftmargin=*]
    \item \looseness -1 An Active Exploration objective that leverages known invariances of the quantity of interest $f$ (Sec. \ref{sec:problem} and \ref{sec:CRL_objective}).
    \item \AlgNameDef, an algorithm that extends MDP homomorphism to Convex RL and solves AE via abstraction (Sec. \ref{sec:algorithm}).
    \item The first analysis that sheds light on the benefits of abstraction on sample efficiency via MDP homomorphisms, here specialized for the AE problem (Sec. \ref{sec:analysis}).
    \item An experimental evaluation of the performance  of \AlgNameShort against a classic AE algorithm (Sec. \ref{sec:experiments}).
\end{itemize}
Our analysis capturing the benefit of geometric structure on sample efficiency may be of independent interest for RL and Convex RL.

\section{Background and Notation}
\label{sec:preliminaries}
\looseness -1 Let $X$ be a set, we denote with $\Delta(X)$ the probability simplex over $X$. We define $[N] \coloneqq \{ 1, \ldots, N \}$, and given a number $x$, we denote $x^+ \coloneqq \max\{1,x\}$.

\subsection{Markovian Processes and Active Exploration}
In the following, we briefly introduce basic RL notions and the Active Exploration (AE) problem.

\mypar{Discrete Markovian Processes}~~~
A (discrete) Controlled Markov Process (CMP) is a tuple $\cmp := (\Sspace, \Aspace, P, \mu)$, where $\Sspace$ is a finite state space ($|\Sspace| = S$), $\Aspace$ is a finite action space ($|\Aspace| = A$), $P: \Sspace \times \Aspace \to \Delta (\Sspace)$ is the transition model, such that $P(s' | s, a)$ denotes the conditional probability of reaching $s' \in \Sspace$ when selecting $a \in \Aspace$ in $s \in \Sspace$, and $\mu : \Delta (\Sspace)$ is the initial state distribution. A CMP $\cmp$ paired with a function  $r: \Sspace \times \Aspace \to \mathbb{R}$, i.e., $\cmp_r \coloneqq \cmp \cup {r}$, is a Markov Decision Process (MDP)~\citep{puterman2014markov}.

\looseness -1 An agent interacting with a CMP starts from an initial state $s_0 \sim \mu$. Then, at every time-step $t$, the agent takes an action $a_t$, collects a reward $r(s_t, a_t)$ (when defined) and transitions to $s_{t + 1} \sim P(\cdot | s_t, a_t)$.
The agent's actions are sampled from a stationary policy $\pi: \Sspace \to \Delta(\Aspace)$ such that $\pi (a | s)$ denotes the conditional probability of $a$ in $s$. 

\mypar{Active Exploration}~~~ 
\looseness -1 In the Active Exploration problem, an agent interacts with a MDP $\cmp_f = (\Sspace, \Aspace, P, \mu, f)$ where $f$ is an unknown and deterministic \emph{quantity of interest} providing a noisy observation $x \sim y(s) = f(s) + \nu(s)$ at state $s$. Here, $\nu$ is a distribution with zero mean and unknown heteroscedastic variance \smash{$\sigma^2(s) \in [0,\sigma^2_\mathrm{max}]$} and $x \in [0,F_{\max}]$. In AE, the agent aims to learn a policy to minimize a measure of uncertainty over $f$ through interactions with $\cmp_f$. Notice that, as a sub-case of AE, $f$ can be interpreted as a reward function that an agent wishes to estimate rather than maximize~\citep{lindner2022active}.

\subsection{Invariances and MDP Homomorphisms}
In the following, we introduce basic concepts from abstract algebra and the notion of MDP homomorphism.

\mypar{Equivalence, Invariance, and Symmetries}~~~ \looseness -1 When a function $f: \mathcal{X} \xrightarrow[]{} \mathcal{Y}$ maps two inputs $x, x'$ to the same value $f(x) = f(x')$, we say that $x$ and $x'$ are $f$-equivalent. The set $[x]$ of all points $f$-equivalent to $x$ is called \emph{equivalence class} of $x$. We say that $f$ is invariant across $[x]$. Consider a transformation operator $L_g: \mathcal{X} \xrightarrow[]{} \mathcal{X}$, where $\group = (\{L_g\}_{g \in G}, \cdot)$ is a group, $G$ is an index set, and $\cdot$ denotes composition. Given a function $f: \mathcal{X} \xrightarrow[]{} \mathcal{Y}$, if $L_g$ satisfies:
\begin{equation}
\label{eq:invariance_eq}
    f(x) = f(L_g[x])  \qquad \forall g \in G, \forall x \in \mathcal{X}
\end{equation}
then we say that $f$ is invariant to $L_g$, and we call $\{L_g\}_{g \in G}$ a set of \emph{symmetries} of $f$.

\mypar{MDP Homomorphisms}~~~An \emph{MDP homomorphism} $h$ is a mapping from an \emph{original} MDP $\cmp_r = (\Sspace, \Aspace, P, \mu, r)$ to an \emph{abstract} MDP $\overline{\cmp}_r = (\overline{\Sspace}, \overline{\Aspace}, \overline{P}, \muAbs, \overline{r})$ defined by a surjective map $h: \Sspace \times \Aspace \to \overline{\Sspace} \times \overline{\Aspace}$. In particular, $h$ is composed of a tuple $(\psi, \{\phi_s \mid s \in \Sspace\})$, where $\psi: \Sspace \to \overline{\Sspace}$ is the state map and $\phi_s: \Aspace \to \overline{\Aspace}$ is the state-dependent action map. These maps are built to satisfy the conditions: 
\begin{align}
    \overline{r}(\psi(s), \phi_s(a)) &= r(s,a) \label{eq:mdp_hom_reward}\\
    \overline{P}(\psi(s')\mid \psi(s), \phi_s(a)) &= \sum_{s'' \in [s']}P(s'' \mid s, a) \label{eq:mdp_hom_dynamics}  
\end{align}
for all $s,s' \in \Sspace, a \in \Aspace$. Moreover, given a state $s$ such that $\sAbs \coloneqq \psi(s)$, we denote the equivalence class of $s$ (and $\sAbs$) induced by $\psi$ as $[s]=[\sAbs] \coloneqq \{ s' \in \Sspace : \psi(s') = \psi(s)\}$ and indicate with $E_s \coloneqq E_{\sAbs} =|[s]|$ its cardinality.

\mypar{Policy Lifting}~~~ Given a MDP homomorphism $h$, the Optimal Value Equivalence Theorem by~\citet{ravindran2001symmetries} states that an optimal policy $\overline{\pi}$ for the abstract MDP can be transformed to an optimal policy $\pi$ for the original MDP via the \emph{lifting} operation:
\begin{equation}
\label{eq:policy_lifting}
    \pi(a|s) \coloneqq \frac{\overline{\pi}(\aAbs|\psi(s))}{|\{a \in \phi_s^{-1}(\aAbs)\}|} \ \quad \forall s \in \Sspace, a \in \phi_s^{-1}(\aAbs) 
\end{equation}

\section{Problem Setting}
\label{sec:problem}
Consider the AE problem in a MDP $\cmp_f = (\Sspace, \Aspace, P, \mu, f)$ with known dynamics $P$, and unknown \emph{quantity of interest} $f: \Sspace \to \mathcal{B} \subset \mathbb{R}$. Typically, an agent aims to minimize an estimation error of $f$ of the form:
\begin{tcolorbox}[colframe=white!, top=2pt,left=2pt,right=2pt,bottom=2pt]
\center \textbf{(Classic) Estimation Error}
\begin{equation}
\label{eq:est_error_eq_classic}
    \xi_n = \frac{1}{S}\sum_{s \in \Sspace} \big| \hat{f}_n(s) - f(s) \big|
\end{equation}
\end{tcolorbox}
\looseness -1 where \smash{$\hat{f}_n(s)$} denotes the empirical estimate of $f(s)$ after $n$ steps in the environment ~\cite{tarbouriech2019active}.

\looseness -1 In this work, we consider the case where $f$ and the dynamics $P$ have convenient geometric structures. In a vast variety of applications, the quantity of interest $f$ is known to have certain group-structured symmetries $L_g: \Sspace \xrightarrow[]{} \Sspace$ and state-dependent action symmetries $K^s_g: \Aspace \to \Aspace$. For all $g \in G, s \in \Sspace, a \in \Aspace$, $f$ and $P$ are invariant according to\footnote{Here we extend $f$ such that $f(s,a) \coloneqq f(s) \: \forall a \in \Aspace$ and consider $y$ to satisfy the same set of invariances as $f$.}
\begin{align}
    f(s,a) &= f(L_g[s], K^s_g[a]) \label{eq:mdp_sym_r}\\
    P(s' \mid s,a) &= P(L_g[s'] \mid L_g[s], K^s_g[a]) \label{eq:mdp_sym_P}
\end{align}
\looseness -1 For instance, in the diffusion process in Fig. \ref{fig:diffusion_drawing}, $f$ follows radial symmetries, while $P$ has roto-translation symmetries as in most physical systems~\citep[Table 1]{derPol2020homomorphic}. 

An MDP with this structure, often denoted as \emph{MDP with symmetries} ~\cite{derPol2020homomorphic}, naturally defines an MDP homomorphism $h=(\psi, \{\phi_s \mid s \in \Sspace\})$ that can be efficiently built, as illustrated in Figure \ref{fig:diffusion_drawing}, by mapping state-action pairs across which $f$ and $P$ are invariant to a unique abstract state-action pair~\cite{derPol2020homomorphic, ravindran2001symmetries}. The main intuition with respect to our estimation process, is that all sets of states $[s]$ across which $f$ is invariant, will map along $\psi$ to an abstract state $\sAbs \coloneqq \psi(s) = \psi(s') \in \SspaceAbs$ $\forall s, s' \in [s]$. 

\looseness -1 We consider the case where such an MDP homomorphism $h$ encoding the underlying geometric structure is known. This is a fair assumption for a large class of applications, where geometric priors can easily be represented via a MDP homomorphism~\cite{derPol2020homomorphic}. Nonetheless, in Section \ref{sec:conclusions} we briefly discuss how the contributions presented in this work can be leveraged in the case of unknown MDP homomorphism. In the following, we introduce a novel geometric estimation error that makes it possible to leverage geometric priors both while learning an optimal sampling strategy, and in inference, when the gathered data is used to compute estimates of the unknown quantity of interest $f$.

\subsection{Geometric Function Estimation}
First, we introduce some quantities updated by the agent when obtaining a noisy realization of $f$ at every time-step $i$, namely $x_i \sim y(s_i)$, where $s_i$ indicates the current state at time-step $i \in [t]$. After $t$ interaction steps we have:
\begin{align}
    T_t(s) &\coloneqq \sum_{i=1}^t \ind\{s_i=s\} \label{eq:visitation_counts}\\
    \hat{f}_t(s) &\coloneqq \frac{1}{T_t^+(s)}\sum_{i=1}^t x_i \ind\{s_i=s\} \label{eq:empirical_mean}\\
    \hat{\sigma}^2_t(s) &\coloneqq \frac{1}{T_t^+(s)}\sum_{i = 1}^t x_i^2 \ind\{s_i=s\} - \hat{f}_t(s)^2 \label{eq:empirical_variance}
\end{align}
which are respectively the \emph{visitation counts}, the \emph{empirical mean} and \emph{empirical variance}. Now, we define the geometric estimation error given $n$ samples from $f$ as follows.
\begin{tcolorbox}[colframe=white!, top=2pt,left=2pt,right=2pt,bottom=2pt]
\center \textbf{Geometric Estimation Error}
\begin{equation}
\label{eq:est_error_eq}
    \bar{\xi}_n = \frac{1}{S}\sum_{s \in \Sspace} \big| \hat{f}^{A}_n(s) - f(s) \big|
\end{equation}
\end{tcolorbox}
\looseness -1 where $\hat{f}^{A}_n(s)$ is an empirical mean obtained by weighted averaging across all states within the same $f$-equivalence class $[s]$. Formally, given \smash{$T_n^+([s]) \coloneqq \sum_{s' \in [s]}T_n^+(s')$}, we have:
\begin{equation}
\label{eq:average_empirical_mean}
    \hat{f}^{A}_n(s) \coloneqq \frac{1}{T_n^+([s])} \sum_{s' \in [s]} T_n(s') \hat{f}_n(s')
\end{equation}
\looseness -1 Interestingly, the geometric estimation error (Eq. \ref{eq:est_error_eq}) generalizes the classic AE estimation error~\cite{tarbouriech2019active}, which corresponds to the limit case of ours where there are no $f$-invariances and therefore every equivalence class is composed of only one state.  

Given the geometric estimation error $\bar{\xi}_n$, any $\epsilon > 0$ and $\delta \in (0, 1)$, we say that an estimate of $f$ is $(\epsilon, \delta)$-accurate if:
\begin{equation}
\label{eq:PAC_RL_problem}
    \PR(\bar{\xi}_n \leq \epsilon) \geq 1-\delta
\end{equation}
\looseness -1 Notice that, in this case, the RL agent is used as an algorithmic tool to estimate an external property of the environment and can be interpreted as an active sampler of the underlying Markov chain.
In particular, we aim to design an algorithm that minimizes the sample complexity needed to estimate $f$ nearly-optimally in high probability.

\begin{restatable}[Sample Complexity Geometric Estimation]{definition}{sample_complexity_def}
\label{def:sample_complexity}
    Given an error $\epsilon>0$ and a confidence level $\delta \in (0,1)$, the sample complexity to solve the geometric function estimation problem is:
    \begin{equation}
        n^{\bar{\xi}}(\epsilon, \delta) \coloneqq \min\{n \geq 1 : \PR(\bar{\xi}_n \leq \epsilon) \geq 1-\delta\}
    \end{equation}
\end{restatable}

\section{From Experimental Design to Convex RL}
\label{sec:CRL_objective}
\looseness -1 In this section, we derive a principled objective to minimize the sample complexity of geometric estimation (Def. \ref{def:sample_complexity}). We first show that $\bar{\xi}_n$ (Eq. \ref{eq:est_error_eq}) can be rewritten as a function of abstract states $\sAbs \in \SspaceAbs$, making it well-defined over the abstract MDP.
\begin{restatable}{proposition}{estErrorRewriting}
\label{proposition:err_rewriting}
    The geometry-aware estimation error $\bar{\xi}_n$ can be rewritten as a function of abstract states as:
    \begin{equation}
        \bar{\xi}_n =  \frac{1}{S}\sum_{\sAbs \in \SspaceAbs} E_{\sAbs} \mid \hat{f}_n(\sAbs) - f(\sAbs) \mid 
    \end{equation}
\end{restatable}
\looseness -1 While Proposition \ref{proposition:err_rewriting} is proved in Appendix \ref{sec:problem_proof}, here we briefly mention the main intuition. Since the empirical estimator $\hat{f}^{A}_n(s)$ aggregates over experiments $s \in [s]$ across which $f$ is invariant, the estimation error can be rewritten by considering only a representative of the $f$-equivalence class $[s]$, namely an abstract state $\sAbs = \psi(s) \in \SspaceAbs$. Then, equality is obtained by reweighting with  the cardinality $E_s$ of $[s]$.
\subsection{Tractable formulation via Convex RL}
Proposition \ref{proposition:err_rewriting} gives $\bar{\xi}_n$ as a function of abstract states, for which we derive the upper bound below (proof in Apx. \ref{sec:problem_proof}).

\begin{restatable}[Convex Upper Bound of $\bar{\xi}_n$]{proposition}{estUpperBound}
    With probability at least $1-\delta$ and $n$ interactions with $f$ we have:
    \begin{align}
    \label{eq:upper_bound_est_err}
        \bar{\xi}_n &\leq \frac{C(n,\overline{S},\delta)}{S}\sum_{\sAbs \in \SspaceAbs} \Fabs(\sAbs; \;T^+_n)
    \end{align}
    with $C(n,\overline{S},\delta) \coloneqq \max \Big \{\log(n\overline{S}/\delta) , \sqrt{\log(n\overline{S}/\delta)} \Big\}$ and
    \begin{align*}
     \quad &
        \Fabs(\sAbs; \;T^+_n) \coloneqq E_{\sAbs}\left(\sqrt{\frac{2\sigma^2(\sAbs)}{T^+_n(\sAbs)}} + \frac{F_{\max}}{T^+_n(\sAbs)}\right) \nonumber
    \end{align*}
    where $T^+_n(\sAbs) \coloneqq T^+_n([s])$ are the visitation counts of $\sAbs$.
\end{restatable}
\looseness -1 As an abstract state represents an $f$-equivalence class of original states, \ie $\sAbs = \psi(s) = \psi(s') \; \forall s,s' \in [\sAbs]$, one can notice that Equation  \ref{eq:upper_bound_est_err} captures two interesting facts. First, equivalence classes $[\sAbs]$ that are under-visited (small $T^+_n(\sAbs)$) with high variance (large $\sigma^2(\sAbs)$) lead to higher estimation error. Second, the cardinality of an equivalence class $E_{\sAbs}$ is proportional to how much its estimation quality impacts the overall estimation error.

\looseness -1 While the upper bound in Equation \ref{eq:upper_bound_est_err} is convex, the constraint set of admissible visitation counts $T^+_n$ is non-convex~\cite{tarbouriech2020active}, rendering this formalization a NP-hard problem~\cite{welch1982algorithmic, tarbouriech2019active}. Nonetheless, problems of this form present a hidden convexity in the asymptotic relaxation ($n \to \infty$) of the dual problem. Given the set $\Lambda$ of admissible asymptotic state-action distributions
\begin{align*}
        \Lambda \coloneqq \{\lambda \in \Delta(\Sspace \times \Aspace) &: \forall s \in \Sspace,\\
    \sum_{b \in \Aspace} \lambda(s,b) &= \sum_{(s',a) \in \Sspace \times \Aspace} P(s|s',a)\lambda(s',a) \}
\end{align*}
we introduce the following $\eta$-smoothened objective.
\begin{tcolorbox}[colframe=white!, top=2pt,left=2pt,right=2pt,bottom=2pt]
\center \textbf{Geometric Estimation Objective}
\begin{equation}
\label{eq:L_infty_obj}
    \Labs_{\infty, \eta}(\lambda) \coloneqq \frac{1}{S} \sum_{\sAbs \in \SspaceAbs} E_{\sAbs} \sqrt{\frac{2\sigma^2(\sAbs)}{ \sum_{s \in [\sAbs]} (\lambda(s) + \eta)}}
\end{equation}
\end{tcolorbox}
Where $\lambda(s) \coloneqq \sum_{a \in \Aspace}\lambda(s, a)$. Crucially, in the following statement, we show that $\Labs_{\infty, \eta}(\lambda)$ is an upper bound of the estimation error $\bar{\xi}_n$ and therefore that  minimizing $\Labs_{\infty, \eta}(\lambda)$ is a principled objective to minimize the sample complexity in Definition \ref{def:sample_complexity}.
\begin{restatable}[Tractable Convex Upper Bound of $\bar{\xi}_n$]{proposition}{tractableEstUpperBound} \label{prop: tractable_convex_upper_bound} Let an empirical state-action frequency at time $t$ be defined as $\lambda_t(s,a) = T_t(s,a)/t$, then for $E_{\overline{s}}\eta \leq \frac{1}{n}$ we have:
\begin{equation}
    \bar{\xi}_n \leq \frac{2S}{\sqrt{n}}C(n, \overline{S}, \delta)\left[\Labs_{\infty, \eta}(\lambda_n) +\frac{\overline{S} F_{\max}}{S \sqrt{n} \eta}\right] 
\end{equation}
\end{restatable}
\looseness -1 Interestingly, in the next section we show that this problem can be solved by computing optimal sampling policies for abstract MDPs only.

\section{\AlgNameDef}
\label{sec:algorithm}
\setlength{\textfloatsep}{6pt}
\begin{algorithm}[t!]
\small
\caption{\AlgNameDef}
\label{alg:gae_algorithm}
    \begin{algorithmic}[1]
        \STATE \textbf{Input:} $\eta$, $h$, $\mathcal{M}$, $\delta$, $\{\tau_k\}_{k \in [K-1]}$
        \STATE Compute abstract CMP $\overline{\cmp}$ induced by $h, \cmp$
        \STATE Initialize $\overline{\lambda}_1 = 1/\Bar{S}\Bar{A}$
        \FOR{$k=1,2,...,K-1$}
        \STATE Compute abstract reward $\Bar{r}^k_{\lambdaAbs_k}$ $\forall \sAbs \in \AspaceAbs, \aAbs \in \AspaceAbs$
        \begin{equation*}
            \Bar{r}^k_{\lambdaAbs_k}(\sAbs, \aAbs) \coloneqq  \frac{- E_s \left[\sqrt{2 \hat{\sigma}_{t_k - 1}^2(\overline{s})} + \alpha(t_k -1, \overline{s}, \delta) \right]}{ 2 S \big(\lambdaAbs_k(\sAbs) + E_{\sAbs} \eta \big)^{\frac{3}{2}}}
        \end{equation*}
        \STATE $\overline{\pi}_{k+1}^+  \xleftarrow{}\mdpSolver\left[\overline{\cmp}_{\Bar{r}}^k =\left(\SspaceAbs, \AspaceAbs, \overline{P}, \muAbs,  \Bar{r}^k_{\lambdaAbs_k}\right) \right]$ 
        \STATE Lift abstract policy
        \begin{equation*}
           \pi^+_{k+1}(a|s) =  \frac{\overline{\pi}_{k+1}^+(\aAbs|\psi(s))}{|\{a \in \phi_s^{-1}(\overline{a})\}|}, \forall s \in \Sspace, a \in \phi_s^{-1}(\aAbs)
        \end{equation*}
        \STATE Deploy policy $\pi^+_{k+1}$ in $\mathcal{M}_f$ for $\tau_k$ steps
        \STATE Compute \smash{$\hat{f}_{t_{k+1}-1}$} and \smash{$\Tilde{\upsilon}_{k+1}$}
        \STATE Aggregate estimates according to \smash{$\overline{\cmp}_f^k$}
        \begin{align*}
            \hat{f}_{t_{k+1}-1}(\overline{s}) &= \frac{1}{T^+_{t_{k+1}-1}(\sAbs)}\sum\limits_{s \in [\sAbs]} T^+_{t_{k+1}-1}(s) \hat{f}_{t_{k+1}-1}(s)\\\Tilde{\overline{\upsilon}}_{k+1}(\overline{s},\aAbs) &= \sum_{s \in [\overline{s}]} \Tilde{\upsilon}_{k+1}(s,a)
        \end{align*}
        \STATE Update the abstract state-action frequency $\overline{\lambda}_{k+1}$
        \begin{equation*}
            \overline{\lambda}_{k+1} = \frac{\tau_k}{t_{k+1} - 1} \Tilde{\overline{\upsilon}}_{k+1} + \frac{t_k - 1}{t_{k+1} - 1} \overline{\lambda}_{k}
        \end{equation*}
        \ENDFOR
        \STATE \textbf{Return: } $\hat{f}_{t_K - 1}$
    \end{algorithmic}
\end{algorithm}
In this section, we introduce \AlgNameDef, an algorithm for AE that leverages the power of abstraction to improve statistical and computational efficiency.
Alike classic AE algorithms~\cite{tarbouriech2019active, hazan2019maxent}, \AlgNameShort is based on a Frank-Wolfe (FW) scheme~\cite{pmlr-v28-jaggi13} that reduces the problem
\begin{equation}
    \label{eq:opt_infty_prob}
      \min_{\lambda \in \Lambda} \Labs_{\infty, \eta}(\lambda)
\end{equation}
to a sequence of $K$ linear programs, each corresponding to a classic MDP $\cmp_f^k$ with reward $r^k_{\lambda}$ defined as
\begin{equation*}
    \nabla \overline{\mathcal{L}}^+_{t_k - 1}(\lambda)_{[s,a]} = \frac{- E_s \left[\sqrt{2 \hat{\sigma}_{t_k - 1}^2(\overline{s})} + \alpha(t_k -1, \overline{s}, \delta) \right]}{ 2 S \big(\sum_{s \in [s]} (\sum_{b \in \mathcal{A}}\lambda(s, b) + \eta) \big)^{\frac{3}{2}}}
\end{equation*}
where the unknown variances are optimistically bounded via a quantity $\alpha(t,\sAbs,\delta)$ according to the following result.
\begin{restatable} [Variance Concentration~\cite{panaganti2022sample}]{lemma}{variance}
\label{lemma:variance_upper_bound}
For all $\overline{s} \in \overline{\mathcal{S}}$, with probability at least $1-\delta$ we have 
\begin{align*}
    \Big|\sqrt{\sigma^2(\overline{s})}-\sqrt{\hat{\sigma}^2_{t}(\overline{s})}\Big| \leq F_{\max}\sqrt{2\frac{\log(2\overline{S}t^2/\delta)}{T_{t}^+(\overline{s})}} \coloneqq \alpha(t,\sAbs, \delta)
\end{align*}
\end{restatable}
\looseness -1 We show that optimistic gradients ($r^k_{\lambda}$) of Equation \ref{eq:opt_infty_prob} satisfy state-action invariances induced by $f$ (proof in Apx. \ref{sec:algorithm_proof}).
\begin{restatable}[Gradient-Reward Invariances]{proposition}{gradInv}
\label{proposition:gradient_invariance}
    \looseness -1If $f$ is invariant over states $s$ and $s'$, then $\forall s,s'\in [s], \forall a,a' \in \mathcal{A}$
    \begin{equation*}
    \nabla_\lambda \overline{\mathcal{L}}^+_{t_k - 1}(\lambda)[s,a] = \nabla_\lambda \overline{\mathcal{L}}^+_{t_k - 1}(\lambda)[s',a'] 
    \end{equation*}
\end{restatable}
Intuitively, Proposition \ref{proposition:gradient_invariance} is due to the fact that the invariances of $f$ propagate to the gradient via $\hat{\sigma}$, because of its definition in Equation \ref{eq:empirical_variance}. 
As a consequence, we can define the optimistic abstract reward as:
\begin{equation}
\label{eq:abstract_reward}
    \Bar{r}^k_{\lambdaAbs}(\sAbs, \aAbs) \coloneqq  \frac{- E_s \left[\sqrt{2 \hat{\sigma}_{t_k - 1}^2(\overline{s})} + \alpha(t_k -1, \overline{s}, \delta)\right]}{ 2 S \big(\lambdaAbs(\sAbs) + E_{\sAbs} \eta \big)^{\frac{3}{2}}}
\end{equation}
where $\lambdaAbs \in \overline{\Lambda} \subseteq \Delta(\SspaceAbs \times \AspaceAbs)$ is an admissible abstract state-action distribution.\footnote{Notice that the set of admissible abstract state-action distributions $\overline{\Lambda}$ can be defined analogously to $\Lambda$.}
Given the reward \smash{$\Bar{r}^k_{\lambdaAbs}$} and the invariances on the dynamics $P$ in Equation \ref{eq:mdp_hom_dynamics}, the MDP \smash{$\cmp_f^k$} at each step $k \in [K]$ of the FW scheme can be solved by computing the optimal policy for the abstract MDP \smash{$\overline{\cmp}_f^k$} and then lifting it back along the MDP homomorphism $h$ (Eq. \ref{eq:policy_lifting}). This observation unlocks the power of abstraction for AE in MDPs, and leads to the \AlgNameShort algorithm, for which we report the pseudocode in Algorithm \ref{alg:gae_algorithm}.\par

\looseness -1 First, \AlgNameShort computes the abstract CMP $\overline{\cmp}$ given the homomorphism $h$ and the original CMP $\cmp$ (line 2). This operation is computationally efficient as we can perform one sweep over $\Sspace\Aspace$ to compute $\SspaceAbs \AspaceAbs$ by applying $\psi$ and $\phi_s$ to original states and actions respectively. The empirical visitation frequency $\overline{\lambda}_1$ is initialized in line 3. At each iteration, \AlgNameShort computes the optimal abstract policy $\overline{\pi}_{k+1}^+$ (line 6), \eg via value iteration, for the abstract MDP \smash{$\overline{\cmp}_{\Bar{r}}^k = \overline{\cmp} \cup \{\Bar{r}^k_{\lambdaAbs_k}\}$},
where \smash{$\Bar{r}^k_{\lambdaAbs_k}$} is an estimate of the optimistic gradient in \eqref{eq:abstract_reward} based on the samples gathered via policy $\pi^+_{k}$ during the previous iteration. Then, it computes an optimal policy for the original MDP, namely $\pi^+_{k+1}$ by lifting the optimal abstract policy (line 7), which is deployed for $\tau_k$ steps (line 8). The gathered samples of $f$ and the state-action visitation counts are used to update the abstract empirical mean of $f$, namely \smash{$\hat{f}_{t_{k+1}-1}(\overline{s})$}, and compute the abstract empirical state-action distribution \smash{$\Tilde{\overline{\upsilon}}_{k+1}$} by aggregating \smash{$\Tilde{\upsilon}_{k+1}$} across states within the same equivalence class (lines 8-9). Then, the empirical visitation frequency $\overline{\lambda}_{k+1}$ is updated to serve for the gradient estimation at the next iteration. \AlgNameShort outputs the aggregated estimates of $f$. 

\mypar{Benefits of Abstraction}~~~Since AE typically entails solving a sequence of MDPs, encoding each instance via a (smaller) abstract MDP (line 6) gives significant computational benefits, as shown in Section \ref{sec:experiments}. From a statistical perspective, the fundamental advantage of \AlgNameShort is leveraging known invariances during the density estimation process (lines 8-10) common in previous works~\cite{hazan2019maxent, tarbouriech2019active}, leading to faster convergence. However, how do different degrees of geometric structure benefit the statistical efficiency of the problem? In the next section, we formally answer this question by presenting a sample complexity result showcasing a geometric compression term.
\looseness - 1 In the following, we denote by sampling strategy the empirical distribution \smash{$\lambda_n$} over $\Sspace \times \Aspace$ induced by the policies \smash{$\{\pi_{k}^+\}_{k \in [K]}$}.

\section{Theoretical Analysis}
\label{sec:analysis}
\looseness -1 In this section, we present an upper bound on the regret and the sample complexity achieved by \AlgNameShort against an optimal sampling strategy. The latter result captures the impact of abstraction on the complexity of the problem via the following notion of geometric compression. 
\begin{tcolorbox}[colframe=white!, top=2pt,left=2pt,right=2pt,bottom=2pt]
\begin{restatable}[Geometric Compression Term]{definition}{geomCompressionTerm}
\label{def:geometric_compression_term}
We denote as geometric compression term $\Phi$ the ratio between the cardinalities of the abstract and original state spaces, formally:
\begin{equation}
    \Phi \coloneqq \overline{S}/S \in (0,1]
\end{equation}
\end{restatable}
\end{tcolorbox}
Before presenting these results, we state two assumptions we employed for deriving them.
\begin{figure*}[t]
\centering
\subcaptionbox[Short Subcaption]{
    Diffusion Deterministic Dynamics
    \label{subfig:diff_stat}}
[
    0.33\textwidth 
]
{%
    \includegraphics[width=0.3\textwidth]{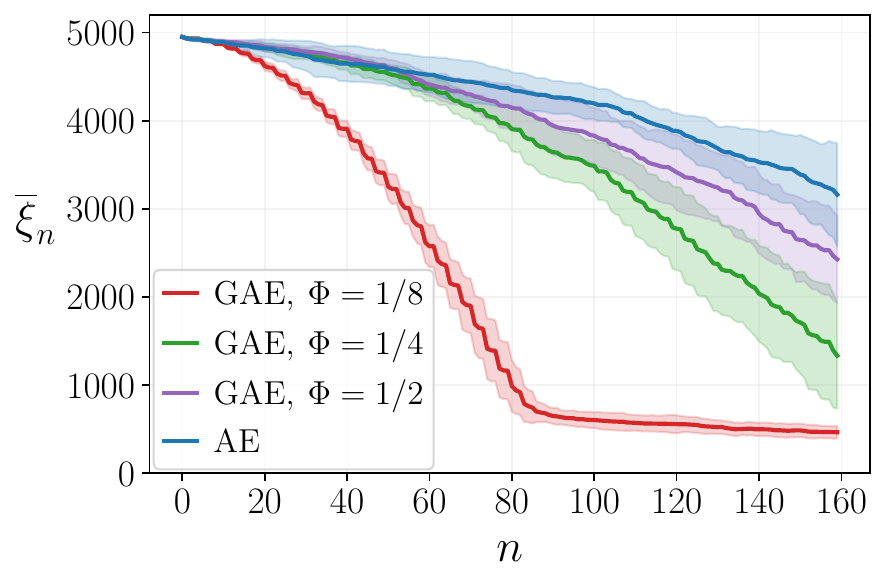}%
}%
\hfill 
\subcaptionbox[Short Subcaption]{%
    Diffusion Stochastic Dynamics (SD)
    \label{subfig:diff_stat_stoch}%
}
[%
    0.33\textwidth 
]%
{%
    \includegraphics[width=0.3\textwidth]{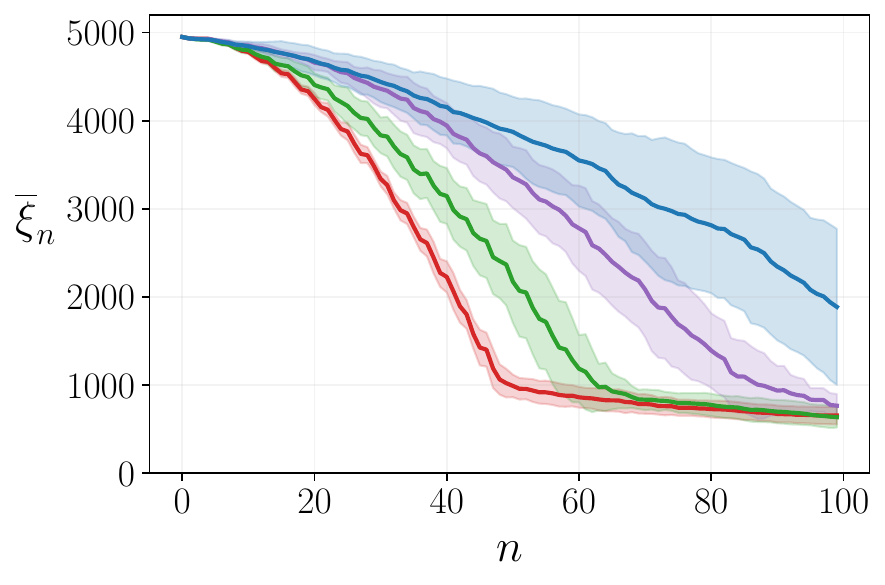}%
}%
\hfill
\subcaptionbox[Short Subcaption]{%
    Diffusion with Inference Bias (SD) %
    \label{subfig:diffusion_interence_bias}%
}
[%
    0.33\textwidth 
]%
{%
    \includegraphics[width=0.3\textwidth]{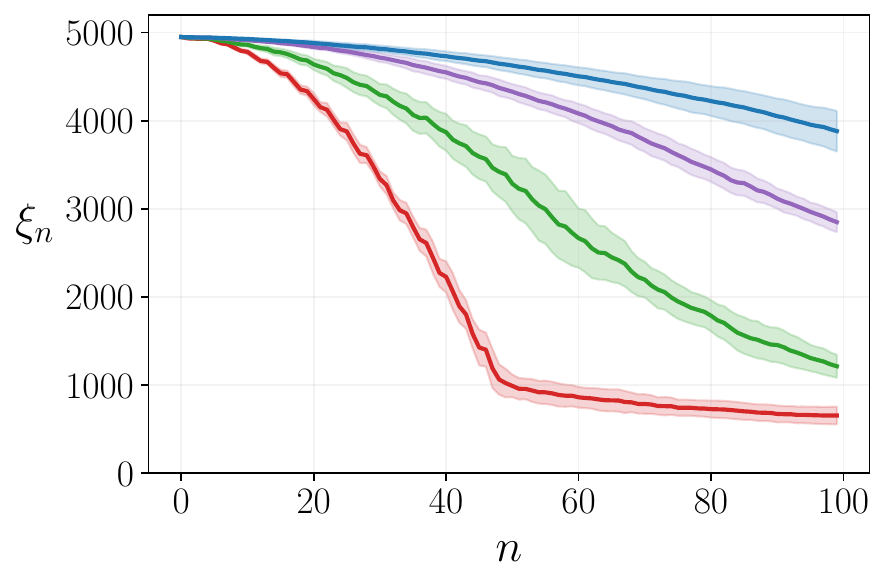}%
}%
\hfill
\subcaptionbox[Short Subcaption]{%
   Diffusion Runtime (SD) %
\label{subfig:diffusion_runtime}%
}
[%
    0.33\textwidth 
]%
{%
    \includegraphics[width=0.3\textwidth]{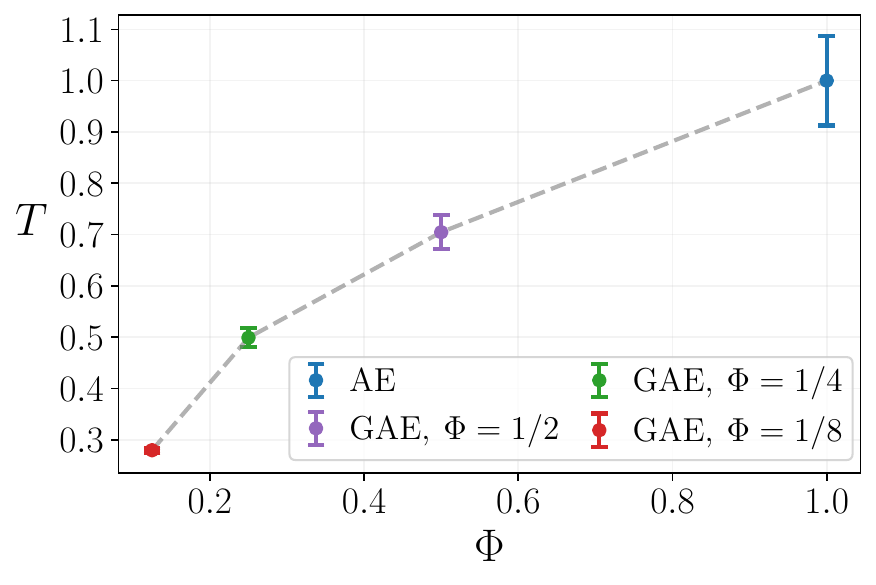}%
}%
\hfill
\subcaptionbox[Short Subcaption]{%
      Strings Statistical Performance%
    \label{subfig:strings_stat}%
}
[%
    0.33\textwidth 
]%
{%
    \includegraphics[width=0.3\textwidth]{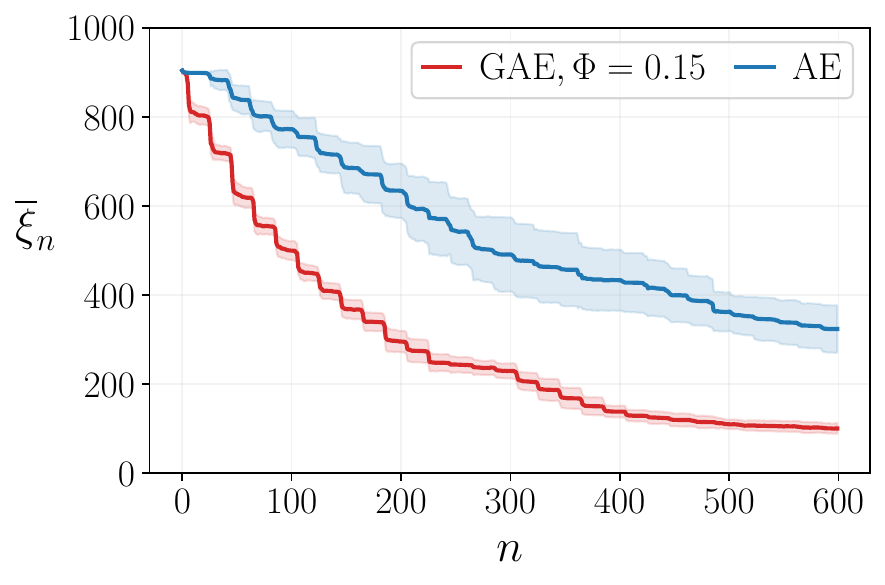}%
}%
\hfill
\subcaptionbox[Short Subcaption]{%
      Chemical compound generation%
    \label{subfig:strings_drawing}%
}
[%
    0.33\textwidth 
]%
{%
    \includegraphics[width=0.3\textwidth]{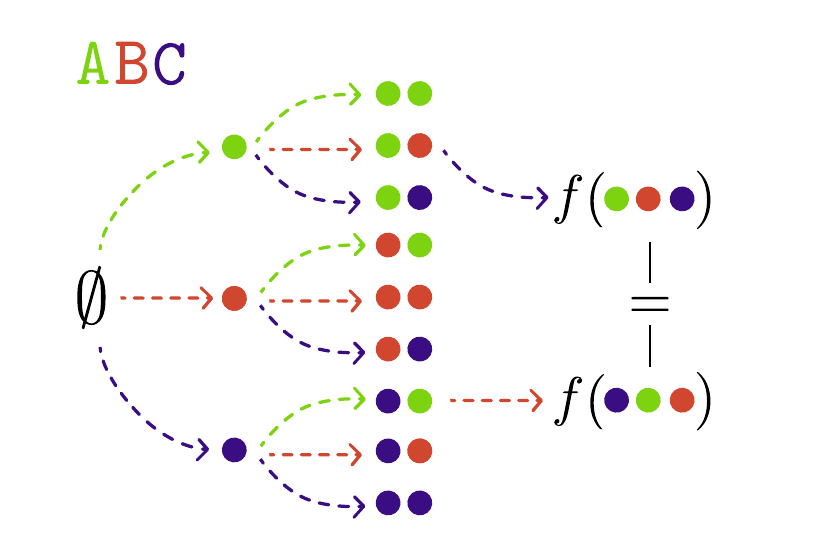}%
}%
\caption[Short Caption]{\looseness -1 Comparison of \AlgNameShort with AE. GAE shows better statistical and computational efficiency. Experiments were carried out over 15 seeds and confidence intervals shown are $\pm$ one standard deviation. \ref{subfig:diff_stat} the statistical advantage of GAE with compression $\Phi$ against AE for deterministic dynamics in the diffusion environment. \ref{subfig:diff_stat_stoch} same setting as \ref{subfig:diff_stat}, but with stochastic dynamics. \ref{subfig:diffusion_interence_bias} the (classic) estimation error taken over the abstract state space. \ref{subfig:diffusion_runtime} computational advantage of \AlgNameShort over AE for different degrees of compression (standardized). \ref{subfig:strings_stat} the statistical advantage of GAE over AE in the strings environment. \ref{subfig:strings_drawing} the strings environment and the invariance of $f$ under permutation.}
\label{fig:label}
\end{figure*}
\begin{restatable}[Homogeneous Equivalence Classes]{assumption}{homogeneousEquivalenceClasses}
\label{assumption:eq_classes_same_cardinality}
The equivalence classes induced by $f$ over $\Sspace$ are homogeneous \ie they have same cardinality, $E_s = E_{s'} \: \forall s, s' \in \Sspace$.
\end{restatable}
Moreover, due to the non-episodic nature of our setting we need to assume the following.
\begin{restatable}[Ergodicity]{assumption}{ergodicity}
The Markov chain induced by any Markovian stationary policy is ergordic.
\end{restatable}
\subsection{Regret Analysis}
Given an empirical state-action distribution $\lambda_n \in \Delta(\Sspace \times \Aspace)$ induced by a sequence of sampling policies interacting with the environment, we define its regret against the optimal sampling strategy as follows
\begin{equation}
    \regret_n(\lambda_n) \coloneqq  \Labs_{\infty, \eta}(\lambda_n) - \Labs_{\infty, \eta}(\lambda^*)
    \label{eq: definition_regret}
\end{equation}
where $\lambda^* \coloneqq \argmin_{\lambda \in \Lambda} \Labs_{\infty, \eta}(\lambda)$ is an optimal state-action distribution. Notice that, while common in AE ~\cite{tarbouriech2019active}, this notion of regret is not standard in RL~\cite{szepesvari2009reinforcement}.
\begin{tcolorbox}[colframe=white!, top=2pt,left=2pt,right=2pt,bottom=2pt]
\begin{restatable}[Regret Guarantee]{theorem}{regretTheorem}
\label{theorem:regret_theorem}
    If algorithm \AlgNameShort is run with a budget of $n$ samples and $\tau_k=3k^2-3k+1$ then w.p. at least $1-\delta$, it holds that:
    \begin{flalign*}
    \regret_n=\widetilde{\mathcal{O}}\left(\left(\frac{\Phi^{\frac{1}{2}} S^{\frac{1}{2}}A F_{\max} \sqrt{\sigma^2_{\max}}}{\eta^{\frac{5}{2}}}\right)\frac{1}{n^{1/3}} \right) 
    \end{flalign*}
\end{restatable}
\end{tcolorbox}
In the following, we present a brief sketch of the proof, while complete derivations are deferred to Apx. \ref{sec:analysis_proof}.

\textbf{Step 1.}~~ 
\looseness -1 We derive the result \wrt the abstract variables via a Frank-Wolfe analysis, taking into account (i) the effect of the optimistic gradient and (ii) the error due to the gap between the empirical and stationary distribution induced by the policy at each iteration~\cite{tarbouriech2019active}.

\textbf{Step 2.}~~
Since the density estimation step of \AlgNameShort is carried out \wrt a distribution defined over $\SspaceAbs \times \AspaceAbs$, we notice that it can be analysed with respect to the abstract variables.

\textbf{Step 3.}~~
Finally, in order to state the result \wrt the original MDP variables, we leverage the \emph{geometric compression} term $\Phi$ (Def. \ref{def:geometric_compression_term}).

\looseness -1 In the following, we report the sample complexity bound capturing the effect of abstraction on learning with high probability a nearly-optimal sampling strategy \wrt the geometric estimation objective~\eqref{eq:L_infty_obj}.
\begin{tcolorbox}[colframe=white!, top=2pt,left=2pt,right=2pt,bottom=2pt]
\begin{restatable}[Sample Complexity of Geometric Estimation Objective]
{theorem}{EstSampleComplexity}
\label{theorem:sample_complexity_result}
If algorithm \AlgNameShort is run with $\tau_k=3k^2-3k+1$, for:
\begin{equation*}
    n=\widetilde{\mathcal{O}}\Bigg(\frac{\Phi^{\frac{3}{2}} S^{\frac{3}{2}}A^3 F_{\max}^3 (\sigma^2_{\max})^{\frac{3}{2}}}{\eta^{\frac{15}{2}}\epsilon^3}\Bigg)
\end{equation*}
samples, then we have that with probability at least $1-\delta$:
\begin{equation*}
     \PR\left(\big|\Labs_{\infty, \eta}(\lambda_n) - \Labs_{\infty, \eta}(\lambda^*)\big|\leq \epsilon \right) \geq 1-\delta 
\end{equation*}
\end{restatable}
\end{tcolorbox}
 \looseness -1 Crucially, setting $\Phi=1$ recovers the case where abstraction is not leveraged by the algorithm.
\subsection{Geometric Compression in MDP with Symmetries}
If the MDP $\cmp_f$ has symmetries (see Eq. \ref{eq:mdp_sym_r} and \ref{eq:mdp_sym_P}), it is possible to make explicit the dependency of $\Phi$ on the cardinality of the state symmetries group $\group = (\{L_g\}_{g \in G}, \cdot)$.
\begin{tcolorbox}[colframe=white!, top=2pt,left=2pt,right=2pt,bottom=2pt]
\begin{restatable}[Compression via Group Cardinality]{proposition}{compressionViaGroupCardinality}
\label{prop:compression_via_group_cardinality}
Given a set of group-structured state symmetries $\group = (\{L_g\}_{g \in G}, \cdot)$ and $\mathrm{Stab}(s) = \mathrm{Stab}(s')$ $\forall s,s' \in \Sspace$ then:
\begin{equation*}
    \Phi = \frac{|\mathrm{Stab}(s)|}{|G|}
\end{equation*}
where $\mathrm{Stab}(s) \coloneqq \{g \in G : L_g[s] = s\}$.
\end{restatable}
\end{tcolorbox}
\looseness -1 Proposition \ref{prop:compression_via_group_cardinality}, which we proved via the Orbit-Stabilizer Theorem~\citep[Theorem C-1.16]{rotman2010advanced}, sheds light on the intuition that a higher number of symmetries leads to a higher degree of compression.

\section{Experiments}
\label{sec:experiments}
\looseness -1 In this section, we perform a thorough experimental evaluation of \AlgNameShort analysing its statistical and computational efficiency on two tasks where the unknown quantity $f$ represents: (1) the amount of pollutant emerging from a point source (see Fig. \ref{fig:diffusion_drawing}), and (2) the toxicity of chemical compounds generated from a set of base elements (see Fig. \ref{subfig:strings_drawing}). In all experiments, unless otherwise specified, we compare the data gathering performances of \AlgNameShort with classic AE, an implementation based on Convex RL and analogous to \AlgNameShort, but not exploiting symmetries~\cite{tarbouriech2019active}. More explicitly, AE is the same as \AlgNameShort (see Alg.~\ref{alg:gae_algorithm}) in the case when the homomorphism $h$ is an identity map and hence $\Phi =1$ and $\overline{\mathcal{M}} = \mathcal{M}$.

\mypar{(1) Pollutant Diffusion Process}~~ \looseness -1 We consider the problem of actively estimating the amount of pollution released in the environment from a point source and following a diffusion process with radial symmetries, as illustrated in Figure \ref{fig:diffusion_drawing} and introduced in Section \ref{sec:preliminaries}. The agent can measure the pollution at 30 different radii and at 8 different angles, resulting in $S = 240$ states, and can select actions $\mathcal{A} = \{\mathrm{in}, \mathrm{out}, \mathrm{clockwise}, \mathrm{anticlockwise}, \mathrm{stay}\}$.
In Figure \ref{subfig:diff_stat}, we show the sample efficiency of \AlgNameShort compared with AE for several values of $\Phi$ in the case of deterministic dynamics. We observe a significant effect of different degrees of compression to the data efficiency of the algorithm. Similarly, Figure \ref{subfig:diff_stat_stoch} shows the same comparison for stochastic dynamics. In Figure \ref{subfig:diffusion_interence_bias}, we show that omitting the inductive bias in the inference step, specifically the absence of weighted averaging across equivalence classes, worsens the performance of AE, thus showing the role of exploiting geometric structure also in inference. Ultimately, in Figure \ref{subfig:diffusion_runtime}, we compare the normalized runtime of \AlgNameShort and AE for several degrees of compression, showing the effect of leveraging geometric structure on computational efficiency and hence practical scalability.

\mypar{(2) Toxicity of Chemical Compounds}~~In this experiment, we consider the problem of actively estimating the toxicity of chemical compounds that can be generated using some base chemical elements. Similar to prior work~\cite{Thiede_2022, dong2022deep}, we associate with states chemical compounds represented as strings (see Fig. \ref{subfig:strings_drawing}). Each character of the string thereby stands for a base chemical element. In our simplified setting, we consider three base elements \texttt{A}, \texttt{B}, \texttt{C} and every compound may consist of at most 5 base elements, resulting in a total of $S=363$ states. The actions correspond to the three base elements and a $\mathrm{stay}$ action. If the agent picks an action that corresponds to a base element, this is appended to the current string, resulting in a new compound to which the agent transitions and gets a noisy observation of its toxicity. 
We consider toxicity to be invariant \wrt string permutations, resulting in $\overline{S} = 55$ states and $\Phi \approx 0.15$. Figure \ref{subfig:strings_stat} shows the statistical performances of \AlgNameShort compared with AE. We observe that even a relatively small compression leads to a significant statistical advantage.

\looseness -1 An extensive discussion of the environments, homomorphisms, and implementation details is deferred to Apx. \ref{sec:experimental_details}.

\section{Related Works}
\label{sec:related_works}
\looseness -1 In the following, we present relevant works in MDP Homomorphisms, Active Exploration, and Convex RL.

\mypar{MDP Homomorphisms}~~~
\looseness -1 \citet{ravindran2001symmetries} were among the first to identify the benefits of solving MDPs via MDP homomorphisms. Recently, these have been extended to Deep RL~\cite{derPol2020homomorphic},  approximate invariances~\cite{ravindran2004approximate, jiang2014improving, ravindran2002model}, and continuous domains~\cite{rezaei2022continuous, biza2018online, zhao2022continuous}. While in this work we have built a fundamental connection between Active Exploration and abstraction via MDP homomorphisms, extending the contributions presented to the mentioned settings, namely Deep RL, approximate abstraction, and continuous domains, is an interesting and relevant direction of future work.

\mypar{Active Exploration}~~~
\looseness -1 The AE problem has been introduced in~\citep{tarbouriech2019active} with a non-episodic setting assuming ergodicity and reversibility of the induced Markov chain. Afterwards, the framework has been extended to perform transition dynamics estimation in high probability~\cite{tarbouriech2020active}. Compared with these works, by smoothening the objective (Sec. \ref{sec:CRL_objective}), we remove the need to solve an LP program at every iteration paving the way for more efficient dynamic programming methods, \eg value iteration~\citep[Chapter 6]{puterman2014markov}. Recently, \citet{mutny2023active} introduced an episodic version of the problem, where the agent can reset its state, but ergodicity is not required, while the unknown quantity $f$ is assumed to be an element of a reproducing kernel Hilbert space with known kernel~\cite{mutny2023active}. While this setting can capture the correlation structure of $f$, it does not leverage known geometric structure of the dynamics and therefore cannot \emph{compress} the original MDP into an abstract one to render the algorithm more scalable from a computational perspective.

\mypar{Convex RL}~~~
The algorithmic scheme presented in this work is an instance of a general framework that has received significant attention recently, generally under the name of \emph{Convex RL}~\cite{hazan2019maxent,zhang2020variational,zahavy2021reward,geist2021concave,mutti2022challenging,mutti2023convex, mutti2022importance}. In this framework, a learning agent interacts with a CMP to optimize an objective formulated through a convex function of the state-action distribution induced by the agent policy. Different choices of this convex function allow to cover several domains of practical interest beyond active exploration, such as pure exploration~\citep[e.g.,][]{hazan2019maxent}, imitation learning~\citep[e.g.,][]{abbeel2004apprenticeship}, and risk-averse RL~\cite{garcia2015comprehensive} among others.

\section{Conclusions}
\label{sec:conclusions}
\looseness -1 In this paper, we presented how abstraction can be leveraged to solve the Active Exploration problem with better statistical complexity and computational efficiency. Before presenting some concluding remarks, we briefly mention a few important discussion points.

\mypar{Beyond Known Geometric Structure}~~
\looseness -1 In a wide set of applications \eg molecular design or environmental sensing, geometric priors on $f$ and $P$ are known and can be easily encoded in an abstract MDP as considered in this work. Nonetheless, for a arguably wider set, geometric structure is not known or it is not human-interpretable. In these cases, one can leverage algorithms that automatically discover symmetries in the environment~\cite{angelotti2021expert, narayanamurthy2008hardness} or directly learn a MDP homomorphism~\cite{mavor2022simple, biza2018online, wolfe2006decision, pmlr-v162-mondal22a}. Interestingly, in both these cases, the \AlgNameShort algorithm can be run with the machine-learned homomorphism.

\mypar{Abstraction in Convex RL}~~ 
The presented algorithmic scheme (Alg. \ref{alg:gae_algorithm}) and theoretical analysis (Sec. \ref{sec:analysis}) are not tied to the AE problem treated in this paper and can be straightforwardly extended to leverage abstraction in a variety of Convex RL application areas, including those mentioned within Section \ref{sec:related_works}.

\mypar{Benefit of Abstraction on Statistical Efficiency}~~
Abstraction, via MDP homomorphisms or close variants, has been leveraged in a large body of works~\cite{rezaei2022continuous, derPol2020homomorphic, van2020plannable, van2021multi, soni2006using, ravindran2003smdp, zhu2022invariant, ravindran2002model, mahajan2017symmetry, ravindran2003smdp} showcasing experimental advantages on the sample complexity in the context of RL. Nonetheless, to the best of our knowledge, this work is the first that formally captures the effect of abstraction via MDP homomorphisms on sample complexity (Sec. \ref{sec:analysis}). Moreover, we believe the main ideas within our analysis can be leveraged to treat a large extent of RL settings.

To summarize, in this work we have presented a principled Active Exploration objective that makes it possible to leverage geometric priors on an unknown quantity $f$ and the system dynamics to solve the estimation problem via an abstraction process, thus increasing statistical and computational efficiency. We introduced an algorithm, \AlgNameLong, which we believe could render Active Exploration more practical for a wide variety of real-world settings. Then, we have presented, to the best of our knowledge, the first analysis capturing the effect of abstraction on the sample complexity of the Active Exploration problem, and more in general in RL. Ultimately, we have performed a thorough experimental evaluation of the proposed method on tasks resembling real-world scientific discovery problems while showing promising performances.

\section*{Acknowledgement}
This publication was made possible by the ETH AI Center doctoral fellowship to Riccardo De Santi. 

We would like to thank Emanuele Rossi, Ossama El Oukili, Massimiliano Viola, Marcello Restelli, and Michael Bronstein for collaborating on a previous version of this project. Moreover, we wish to thank Mojm\'ir Mutn\'y and Manish Prajapat for the insightful discussions. 

The project has received funding from the European Research Council (ERC) under the European
Union’s Horizon 2020 research, innovation program grant agreement No 815943 and the Swiss
National Science Foundation under NCCR Catalysis grant number 180544.

\section*{Impact Statement}
This paper presents work whose goal is to advance the field of Reinforcement Learning. There are many potential societal consequences of our work, none which we feel must be specifically highlighted here.

\bibliography{biblio}
\bibliographystyle{icml2024}

\newpage
\onecolumn
\begin{appendix}
\section{List of symbols}
\label{sec:notation}
    \begin{longtable}{lll}
        \multicolumn{3}{c}{\underline{\textbf{General Mathematical Objects}}} \\ 
        \noalign{\vskip 1mm}
         $\Delta(X)$ & $\triangleq$ & Probability simplex over $X$ \\
         $x^+$ & $\triangleq$ & $\max\{1,x\}$ \\
         $\overline{X}$ & $\triangleq$ & Abstract counterpart of variable $X$\\
        \noalign{\vskip 3mm}
        \multicolumn{3}{c}{\underline{\textbf{MDP}}} \\ 
        \noalign{\vskip 1mm}
         $\cmp_r$ & $\triangleq$ & Markov decision process $\cmp_r = (\Sspace, \Aspace,  P, \mu, r)$\\
         $\Sspace$ & $\triangleq$ & State space \\
         $\Aspace$ & $\triangleq$ & Action space \\
         $P$ & $\triangleq$ & Transition model $P : \Sspace \times \Aspace \to \Delta (\Sspace)$ \\
         $r$ & $\triangleq$ & Scalar state function \eg reward $r : \Sspace \times \Aspace \to \Delta ([0, R])$ \\
         $\mu$ & $\triangleq$ & Initial state distribution $\mu \in \Delta (\Sspace)$ \\
         $S$ & $\triangleq$ & Size of the state space $S = |\Sspace|$ \\
         $A$ & $\triangleq$ & Size of the action space $A = |\Aspace|$ \\
         $s$ & $\triangleq$ & State $s \in \Sspace$ \\
         $a$ & $\triangleq$ & Action $a \in \Aspace$ \\
         $\pi$ & $\triangleq$ & Stationary Markovian policy $\pi: \Sspace \to \Delta(\Aspace)$\\
         $\lambda$ & $\triangleq$ & Stationary state-action distribution $\lambda \in \Lambda$\\
         $\Lambda$ & $\triangleq$ & Set of admissible state-action distribution, expression within Section \ref{sec:problem}\\
         \noalign{\vskip 3mm}
         \multicolumn{3}{c}{\underline{\textbf{Symmetries, MDP Homomorphisms}}} \\
         \noalign{\vskip 1mm}
         $L_g$ & $\triangleq$ & state transformation or symmetry $L_g:\Sspace \to \Sspace$ \\
         $K_g^s$ & $\triangleq$ & state-dependent action transformation or symmetry $L_g:\Aspace \to \Aspace$ \\
         $\mathbb{G}$ & $\triangleq$ & Algebraic group structure \eg $\group = (\{L_g\}_{g \in G}, \cdot)$\footnote{we do not explicitly specify the identity and inverse elements of the group as we do not use them in the following} \\
         $g$ & $\triangleq$ & Index of group element $g \in G$ \\
         $[s]$ & $\triangleq$ & Equivalence class of $f$-invariant states $[s] \subseteq \Sspace$\\
         $h$ & $\triangleq$ & MDP homomorphism between $\cmp_f$ and $\overline{\cmp}_f$, $h: \Sspace \times \Aspace \to \SspaceAbs \times \AspaceAbs$, $h = (\psi, \{\phi_s \mid s \in \Sspace\})$\\
         $\psi$ & $\triangleq$ & Homomorphism state map $\psi: \Sspace \to \SspaceAbs$\\
         $\phi$ & $\triangleq$ & Homomorphism state-dependent action map $\phi_s:\Aspace \to \AspaceAbs$\\
         $[\sAbs]$ & $\triangleq$ & Equivalence class of $f$-invariant states mapping to $\sAbs$ along $\psi$, $[\sAbs]=\{ s' \in \Sspace : \psi(s') = \psi(s)\}$ \\
         $E_s$ & $\triangleq$ & Cardinality of equivalence class $[s]$, $E_s = |[s]| = |[\sAbs]| = E_{\sAbs}$ \\
         $\Phi$ & $\triangleq$ & Geometric compression coefficient, $\Phi = \overline{S} / S$\\
         \noalign{\vskip 3mm}
         \multicolumn{3}{c}{\underline{\textbf{\AlgNameShort Algorithm}}} \\
         \noalign{\vskip 1mm}
         $f$ & $\triangleq$ & Unknown state quantity $f: \Sspace \to \mathcal{B} \subset \mathbb{R}$\\
         $\nu$ & $\triangleq$ & Noise random variable with zero mean and unknown variance $\sigma^2$\\
         $y$ & $\triangleq$ & Noisy random variable $y(s) = f(s) + \nu(s) $ \\
         $T_t(s)$ & $\triangleq$ & Visitation counts for state $s$ after $t$ steps, see Equation \eqref{eq:visitation_counts}\\
         $\hat{f}_t(s)$ & $\triangleq$ & Empirical mean for state $s$ after $t$ steps, see Equation \eqref{eq:empirical_mean}\\
         $\hat{\sigma}^2_t(s)$ & $\triangleq$ & Empirical variance for state $s$ after $t$ steps, see Equation \eqref{eq:empirical_variance}\\
          $\hat{f}^A_t(s)$ & $\triangleq$ & Average empirical mean for equivalence class [s], see Equation \eqref{eq:average_empirical_mean}\\
          $\hat{f}_t(\sAbs)$ & $\triangleq$ & Empirical mean for abstract state $\sAbs$ after $t$ steps\\
         $\bar{\xi}_n$ & $\triangleq$ & Geometric estimation error, see Equation \eqref{eq:est_error_eq} \\
         $\epsilon$ & $\triangleq$ & Controllable approximation error for PAC guarantees \\
         $\delta$ & $\triangleq$ & Controllable probability of error for PAC guarantees \\
         $n^{\bar{\xi}}(\epsilon, \delta)$ & $\triangleq$ & Sample complexity for PAC estimation of geometric estimation error $\bar{\xi}_n$, see def. \ref{def:sample_complexity}\\
         $n$ & $\triangleq$ & Sample complexity for PAC optimality of geometric estimation objective, see def. \ref{theorem:sample_complexity_result}\\
         $\Labs_n$ & $\triangleq$ & Finite-samples Convex RL Objective \eqref{eq:finite_samples_obj} \\
         $\Labs_{\infty, \eta}$ & $\triangleq$ & Asymptotic and smoothened Convex RL Objective \eqref{eq:opt_infty_prob} \\
         $\Bar{r}^k_{\lambdaAbs_k}$ & $\triangleq$ & Abstract reward estimated at iteration $k$ via empirical density $\lambdaAbs_k$ \\
         $\overline{\pi}_{k+1}^+$ & $\triangleq$ & Optimal abstract policy \wrt MDP $\overline{\cmp}_{\Bar{r}}^k$ \\
         $\pi^+_{k+1}$ & $\triangleq$ & Optimal original policy at iteration $k$ obtained by lifting $\overline{\pi}_{k+1}$ \\
         $\tau_k$ & $\triangleq$ & Length of trajectory of policy deployed at iteration $k$\\
         $\Tilde{\upsilon}_{k+1}$ & $\triangleq$ & Empirical state-action distribution induced during iteration $k$\\
         $\Tilde{\overline{\upsilon}}_{k+1}$ & $\triangleq$ & Abstract state action distribution obtained at iteration $k$ by aggregating $\Tilde{\upsilon}_{k+1}$ \\
         $\overline{\lambda}_{k+1}$ & $\triangleq$ & Updated state-action frequency at end of $k$-th iteration\\
         \noalign{\vskip 3mm}
         \multicolumn{3}{c}{\underline{\textbf{Regret and Sample Complexity Analysis}}} \\
         \noalign{\vskip 1mm}
         $\overline{\upsilon}^+_{k+1}$ & $\triangleq$ & state-action distribution induced by $\overline{\pi}_{k+1}^+$\\
         $\overline{\lambda}^*$ & $\triangleq$ & optimal abstract state-action distribution of the learning problem\\
         $\lambda^*$ & $\triangleq$ & optimal state-action distribution of the learning problem\\
         $\overline{\upsilon}^*_k$ & $\triangleq$ & optimal state-action distribution of the MDP $\overline{\mathcal{M}}^k_r$\\
         $\nabla \widehat{\mathcal{L}}_{t_k-1}$ & $\triangleq$ & empirical gradient of objective $\overline{\mathcal{L}}_{\infty, \eta}$\\
         $\nabla \widehat{\mathcal{L}}^+_{t_k-1}$ & $\triangleq$ & empirical optimistic gradient of objective $\overline{\mathcal{L}}_{\infty, \eta}$\\
    \end{longtable}
\newpage

\section{Proofs Section \ref{sec:problem}}
\label{sec:problem_proof}
\estErrorRewriting*
\begin{proof}
    By the definition of $\Bar{\xi}_n$ we have
    \begin{align}
        \Bar{\xi}_n &\coloneqq \frac{1}{S}\sum_{s \in \Sspace} |\hat{f}_n^{A}(s) - f(s)|\\
        &\labelrel={step:est_rew_1} \frac{1}{S}\sum_{s \in \Sspace} \Bigg| \frac{1}{T_n^+(\bar{s})} \sum_{s' \in [s]} T_n(s') \hat{f}_n(s') - f(s) \Bigg|\\
        &\labelrel={step:est_rew_2} \frac{1}{S}\sum_{\sAbs \in \SspaceAbs}|[\sAbs]| \Bigg| \frac{1}{T_n^+(\bar{s})} \sum_{s' \in [\sAbs]} T_n(s') \hat{f}_n(s') - f(\sAbs) \Bigg|\\
        &\labelrel={step:est_rew_3} \frac{1}{S}\sum_{\sAbs \in \SspaceAbs}E_{\sAbs} \big| \hat{f}_n(\sAbs) - f(\sAbs) \big| 
    \end{align}
where in step \eqref{step:est_rew_1} we used that $\hat{f}^A_n(s) \coloneqq \frac{1}{T_n^+([s])} \sum_{s' \in [\sAbs]} T_n(s') \hat{f}_n(s')$, in step \eqref{step:est_rew_2} we used $f$ invariances and in step \eqref{step:est_rew_3} we used that $E_{\sAbs} \coloneqq |[\sAbs]|$ and $\hat{f}_n(\sAbs) \coloneqq \frac{1}{T_n^+(\bar{s})} \sum_{s \in [\sAbs]} T_n(s) \hat{f}_n(s)$.
\end{proof}
\estUpperBound*
\begin{proof}
To prove the statement, we employ a Bernstein type inequality from Lemma 7.37 in ~\cite{lafferty2008concentration}, where we upper bound $\sfrac{2}{3}$ by 1 in the second summand. Then, given a $\delta'$, we have that for all $t \in [n]$:
    \begin{equation}
    \label{eq:local_bernstein_f}
        \mathbb{P} \left( \Big|\hat{f}_{t}(\overline{s})-f(\overline{s}) \Big|\leq \sqrt{\frac{2\sigma^2(\overline{s}) \log(1/ \delta')}{T^+_{n}(\overline{s})}}+\frac{F_{\max}\log(1/\delta')}{T^+_{n}(\overline{s})} \right) \geq 1- \delta'
    \end{equation}
    where $T^+_{n}(\overline{s})= \sum_{s \in [\overline{s}]} T^+_{n}(s)$ and $\hat{f}_{t}(\overline{s})=\frac{1}{T^+_t(\sAbs)}\sum_{s \in [\overline{s}]}T_{t}(s)\hat{f}_{t}(s)$. For notational simplicity, we will define:
 \begin{align*}
     B_{\delta'}(\overline{s})=\sqrt{\frac{2\sigma^2(\overline{s}) \log(1/ \delta')}{T^+_{n}(\overline{s})}}+\frac{F_{\max}\log(1/\delta')}{T^+_{n}(\overline{s})}
 \end{align*}
 Then, by using a standard union bound, over $\overline{s} \in \overline{\mathcal{S}}$ and $t \in [n]$:
 \begin{align}
      \mathbb{P} \left( \bigcap_{t \in [n]}\bigcap_{\overline{s} \in \SspaceAbs}\bigg \{\Big|\hat{f}_{t}(\overline{s})-f(\overline{s}) \Big| \leq B_{\delta'}(\overline{s}) \bigg \}\right) &= 1-\mathbb{P} \left( \bigcup_{n \in [t]}\bigcup_{\overline{s} \in \SspaceAbs}\bigg \{\Big|\hat{f}_{t}(\overline{s})-f(\overline{s}) \Big| > B_{\delta'}(\overline{s}) \bigg \}\right) \nonumber\\
      &\geq 1-\sum_{t \in [n]}\sum_{\overline{s} \in \SspaceAbs} \mathbb{P} \left( \Big|\hat{f}_{t}(\overline{s})-f(\overline{s}) \Big| > B_{\delta'}(\overline{s}) \right) \nonumber\\ &\labelrel\geq{new_step:mid_upper_bound_1}1-n \overline{S} \delta' \label{eq:high_prob_bound_1}
 \end{align}
where in step \eqref{new_step:mid_upper_bound_1} we have used Equation \eqref{eq:local_bernstein_f}. Then:
 \begin{align}
     1-\delta &\labelrel\leq{step:Prop_2_step_1} \mathbb{P} \left( \bigcap_{t \in [n]}\bigcap_{\overline{s} \in \SspaceAbs}\bigg \{\Big|\hat{f}_{t}(\overline{s})-f(\overline{s}) \Big| \leq B_{\delta'}(\overline{s}) \bigg \}\right) \nonumber\\
     &= \mathbb{P} \left( \bigcap_{t \in [n]}\bigcap_{\overline{s} \in \SspaceAbs}\bigg \{E_{\sAbs} \Big|\hat{f}_{t}(\overline{s})-f(\overline{s}) \Big| \leq E_{\sAbs}B_{\delta'}(\overline{s}) \bigg \}\right)
     \label{eq:prop2_eq_1}\\
     &\leq \mathbb{P} \left(\bigcap_{\overline{s} \in \SspaceAbs}\bigg \{E_{\sAbs} \Big|\hat{f}_{n}(\overline{s})-f(\overline{s}) \Big| \leq E_{\sAbs}B_{\delta'}(\overline{s}) \bigg \}\right)
     \label{eq:prop2_eq_1}\nonumber\\
     &\labelrel\leq{step:Prop_2_step_2} \mathbb{P} \left( \sum_{\overline{s} \in \SspaceAbs}E_{\sAbs} \Big|\hat{f}_{n}(\overline{s})-f(\overline{s}) \Big| \leq \sum_{\overline{s} \in \overline{S}} E_{\sAbs} B_{\delta'}(\overline{s}) \right) \nonumber\\
     &=\mathbb{P}\left( S\Bar{\xi}_n \leq \sum_{\overline{s} \in \SspaceAbs} E_{\sAbs} B_{\delta'}(\overline{s})\right)\nonumber\\
     &= \mathbb{P}\left( \Bar{\xi}_n \leq \frac{1}{S}\sum_{\overline{s} \in \SspaceAbs} E_{\sAbs} B_{\delta'}(\overline{s}) \right)\nonumber
 \end{align}

where in \eqref{step:Prop_2_step_1} we used Equation \ref{eq:high_prob_bound_1} and in step \eqref{step:Prop_2_step_2} we used the trivial fact that since the inequality in Equation \ref{eq:prop2_eq_1} holds for every $\sAbs \in \SspaceAbs$ therefore it holds also for the respective sums over $\SspaceAbs$. Then, we get that with probability at least $1-\delta$:
\begin{align*}
    \Bar{\xi}_n & \leq \frac{1}{S}\sum_{\overline{s} \in \SspaceAbs}E_{\sAbs} \bigg( \sqrt{\frac{2\sigma^2(\overline{s}) \log(n\overline{S}/ \delta)}{T^+_{n}(\overline{s})}}+\frac{F_{\max}\log(n\overline{S}/\delta)}{T^+_{n}(\overline{s})} \bigg)\\ 
    &\leq \frac{1}{S} \max \Big \{\log(n\overline{S}/\delta) , \sqrt{\log(n\overline{S}/\delta)} \Big \} \sum_{\overline{s} \in \SspaceAbs} E_{\sAbs} \bigg( \sqrt{\frac{2\sigma^2(\overline{s})}{T^+_{n}(\overline{s})}}+\frac{F_{\max}}{T^+_{n}(\overline{s})} \bigg)\\
    &=\frac{1}{S}  \max \Big \{\log(n\overline{S}/\delta) , \sqrt{\log(n\overline{S}/\delta)} \Big \} \sum_{\overline{s} \in \SspaceAbs}\Fabs(\sAbs; \;T^+_n)\\
    &=\frac{C(n, \overline{S}, \delta)}{S} \sum_{\overline{s} \in \SspaceAbs}\Fabs(\sAbs; \;T^+_n)
\end{align*}
\end{proof}
\tractableEstUpperBound*
\begin{proof}
First, we define the following:
\begin{equation}
    \Fabs(\sAbs \; ; \lambda) \coloneqq E_{\sAbs}\bigg(\Fabs_1(\sAbs \; ; \lambda) + \Fabs_2(\sAbs \; ; \lambda)\bigg) \label{eq:convex_obj_finite_samples}
\end{equation}
with
\begin{align*}
    \Fabs_1(\sAbs \; ; \lambda) &= \sqrt{\frac{2\sigma^2(\sAbs)}{\sum_{a \in \Aspace} \sum_{s \in [\sAbs]} \lambda(s, a) + \frac{1}{n}}}\\
    \Fabs_2(\sAbs \; ; \lambda) &= \frac{1}{\sqrt{n}}\frac{F_{\max}}{\sum_{a \in \Aspace} \sum_{s \in [\sAbs]} \lambda(s, a) + \frac{1}{n}}
\end{align*}
and the following auxiliary objective:
\begin{equation}
\label{eq:finite_samples_obj}
    \Labs_n(\lambda) \coloneqq \frac{1}{S} \sum_{\sAbs \in \SspaceAbs} \Fabs(\sAbs \; ; \lambda)
\end{equation}

Then, given an empirical state-action frequency at time $t$, defined as $\lambda_t(s,a) = T_t(s,a)/t$, we have that:
 \begin{align*}
     \bar{\xi}_n &\labelrel\leq{step:Prop_3_step_1} \frac{2}{\sqrt{n}}C(n, \overline{S}, \delta) \sum_{\sAbs \in \SspaceAbs} \Fabs(\sAbs \; ; \lambda_n) \\
     &\labelrel={step:Prop_3_step_2}\frac{2S}{\sqrt{n}}C(n, \overline{S}, \delta)\Labs_n(\lambda_n)\\
     &\leq\frac{2S}{\sqrt{n}}C(n, \overline{S}, \delta)\left[\Labs_{\infty, \eta}(\lambda_n) +\ \frac{\overline{S} F_{\max}}{S \sqrt{n} \eta}\right]
 \end{align*}
where in step \eqref{step:Prop_3_step_1} we employed Equations \ref{eq:convex_obj_finite_samples} and \ref{eq:upper_bound_est_err} and the fact that $T^+_n(s) = \max(T_n(s), 1) \geq (T_n(s) + 1)/2$. In step \eqref{step:Prop_3_step_2} we used the definition of the objective in Equation \ref{eq:finite_samples_obj} and in the last inequality we used Lemma \ref{lemma:bound_L_n}.
\end{proof}
\newpage

\section{Proofs Section \ref{sec:algorithm}}
\label{sec:algorithm_proof}
\variance*

\begin{proof}
    From \citep{panaganti2022sample}, we have that for a fixed time $k \leq t$, it holds that with probability at least $1-\tilde{\delta}$:

    \begin{align*}
    \Big|\sqrt{\sigma^2(\overline{s}_k)}-\sqrt{\hat{\sigma}^2_{t}(\overline{s}_k)}\Big| \leq F_{\max}\sqrt{2\frac{\log(2\overline{S}/\delta)}{T_{t}^+(\overline{s})}} 
\end{align*}

Since we want the result to hold $\forall k \leq t$, we simply union bound over time which leads to the desired result.
\end{proof}

\gradInv*

\begin{proof}
    The proof simply follows by computing the gradient explicitly. In particular, we have that
    \begin{flalign*}
        \nabla_\lambda \widehat{\mathcal{L}}^+_{t_k - 1}(\lambda)_{[s,a]} =\frac{- E_s \Big[ \sqrt{2 \hat{\sigma}_{t_k - 1}^2(\overline{s})} + \alpha(t_k -1, \overline{s}, \delta)\Big]}{ 2 S \big( \sum_{s \in [s]}\sum_{b \in \mathcal{A}}\lambda(s, b) +E_{\overline{s}} \ \eta \big)^{\frac{3}{2}}}=\nabla_\lambda \widehat{\mathcal{L}}^+_{t_k - 1}(\lambda)_{[s',a']}
    \end{flalign*}
    Indeed, we notice that the gradient does not directly depend on the actions $a$ and $a'$ since in the denominator we are summing over all the possible actions. Furthermore, since both $s, s' \in [s]$, the gradient remains unchanged since it depends exclusively on the abstract state $\overline{s}$ and not directly on $s$ and $s'$. In particular, since $s$ and $s'$ are in the same equivalence class, then the corresponding abstract state $\overline{s}$ is the same.
    \end{proof}
\newpage

\section{Proofs Section \ref{sec:analysis}}
\label{sec:analysis_proof}
\regretTheorem*
\begin{proof}

The analysis follows a classic Frank-Wolfe (FW) scheme analysis while taking into account the approximation error due to the optimistic estimate of the gradient, and the estimation error due to the gap between the asymptotic distribution associated with $\overline{\pi}^+_{k+1}$ and the induced empirical frequences $\overline{\lambda}_{k+1}$ at each iteration $k$.  It diverges from previous analysis for non-episodic AE settings~\citep{tarbouriech2019active, tarbouriech2020active} in two main ways: (i) by leveraging ergodicity and hence uniqueness of the stationary distribution induced by Markovian stationary policies, we study the density estimation process via~\citep[Lemma 5]{mutny2023active}, (ii) we introduce a dependency on the geometric compression term $\Phi$ (Definition \ref{def:geometric_compression_term}) in order to show the effect of compression on the final sample complexity result in Theorem \ref{theorem:sample_complexity_result}.

As the regret in Convex RL is interpreted as a suboptimality gap~\cite{mutti2023convex,mutti2022challenging, tarbouriech2019active}, we first derive an upper bound on the approximation error, as defined in Equation \ref{eq: definition_regret}, achieved at the end of iteration $k\in [K]$ of Algorithm \ref{alg:gae_algorithm}. In the following we denote with $t_{k+1}$ the number of samples gathered until the end of iteration $k$, formally $t_{k+1}:=\sum_{j=1}^{k} \tau_j$, where $\tau_j$ is the number of steps policy $\pi^+_{k+1}$ has been released in iteration $k$. Then by defining $\mathcal{L} \coloneqq \Labs_{\infty, \eta}$, we derive the following.
\begin{align}
    \rho_{k+1}&:=\mathcal{L}(\overline{\lambda}_{k+1})-\mathcal{L}(\overline{\lambda}^*) \label{eq:regret_eq_1}\\
    &=\mathcal{L}((1-\beta_k)\overline{\lambda}_k+\beta_k \Tilde{\overline{\upsilon}}_{k+1})-\mathcal{L}(\overline{\lambda}^*) \tag{where $\beta_k \coloneqq \tau_k /(t_{k+1} - 1)$}\nonumber\\
    &\labelrel\leq{step:regret_1} \mathcal{L}(\overline{\lambda}_{k})-\mathcal{L}(\overline{\lambda}^*)+\beta_k \langle \nabla \mathcal{L}(\overline{\lambda}_{k}), \Tilde{\overline{\upsilon}}_{k+1}-\overline{\lambda}_{k} \rangle + C_{\eta}\beta_k^2 \nonumber\\
    &=\mathcal{L}(\overline{\lambda}_{k})-\mathcal{L}(\overline{\lambda}^*)+\beta_k \langle \nabla \mathcal{L}(\overline{\lambda}_{k}), \overline{\upsilon}^*_{k+1}-\overline{\lambda}_{k}\rangle+\beta_k \langle \nabla \mathcal{L}(\overline{\lambda}_{k}), \Tilde{\overline{\upsilon}}_{k+1}-\overline{\upsilon}^*_{k+1}\rangle + C_{\eta}\beta_k^2 \nonumber\\
    &\labelrel\leq{step:regret_2} \mathcal{L}(\overline{\lambda}_{k})-\mathcal{L}(\overline{\lambda}^*)+\beta_k \langle \nabla \mathcal{L}(\overline{\lambda}_{k}), \overline{\lambda}^*-\overline{\lambda}_k\rangle+\beta_k \langle \nabla \mathcal{L}(\overline{\lambda}_{k}), \Tilde{\overline{\upsilon}}_{k+1}-\overline{\upsilon}^*_{k+1}\rangle + C_{\eta}\beta_k^2 \ \nonumber\\
    &\labelrel\leq{step:regret_3} (1-\beta_k)\rho_k+C_{\eta}\beta_k^2+\beta_k \underbrace{\langle \nabla \mathcal{L}(\overline{\lambda}_{k}), \overline{\upsilon}^+_{k+1}-\overline{\upsilon}^*_{k+1}\rangle}_{\epsilon_{k+1}}+\beta_k\underbrace{\langle \nabla \mathcal{L}(\overline{\lambda}_{k}), \Tilde{\overline{\upsilon}}_{k+1}-\overline{\upsilon}^+_{k+1}\rangle}_{\Delta_{k+1}} \label{eq: last_eq_tracking_error}
\end{align}
where in step \eqref{step:regret_1} we use $C_{\eta}$-smoothness of $\mathcal{L}$, in step \eqref{step:regret_2} we use the definition of the update step of FW and the fact that $\overline{\upsilon}^*_{k+1}$ is optimal and in \eqref{step:regret_3} we use again $C_{\eta}$-smoothness to bound $\langle \nabla \mathcal{L}(\Tilde{\lambda}_{k}), \lambda^*-\Tilde{\lambda}_k\rangle$. Note that in the first term $\epsilon_{k+1}$, we take into account the discrepancy between the state-action distribution $\overline{\upsilon}^*_{k+1}$ induced by the optimal policy \wrt the MDP with the exact gradient as reward, and our exact solution $\overline{\upsilon}^+_{k+1}$ of the MDP with the optimistic gradient as reward. In the second term $\Delta_{k+1}$, we take into account the error due to the gap between the state-action distribution $\overline{\upsilon}^+_{k+1}$ and the empirical distribution $\Tilde{\overline{\upsilon}}_{k+1}$ induced by deploying policy $\pi^+_{k+1}$ for $\tau_k$ steps. In the following, we upper bound independently the terms $\Delta_{k+1}$ and $\epsilon_{k+1}$.

\subsection{Upper Bound $\Delta_{k+1}$}
We derive a PAC guarantee on the density estimation error by using Lemma \ref{PAC guarantee} with $\eta_t(\overline{s})=\eta(\overline{s})=\overline{\lambda}(\overline{s})$, as in our case $\overline{\lambda}$ is fixed within one FW iteration and set $f^{\overline{s}}(\cdot)=f_t^{\overline{s}}(\cdot) := \frac{1}{T} \langle \cdot, \delta_{\overline{s}} \rangle$, where $\delta_{\overline{s}}$ is the vector of all zeros except in position $s$ where it has a one and where $\langle \cdot, \cdot \rangle$ denotes the inner product. In particular, in this case, $f^{\overline{s}}$ corresponds to the evaluation functional of a probability distribution in state $\overline{s}$. Note that this functional is clearly linear as requested by the proposition as the inner product is linear. Furthermore, we can notice that $||f_{\overline{s}}||_{\infty}=\frac{1}{T}$ and hence we can restate the bound presented in the Lemma as follows:
\begin{align*}
    \left| \frac{\sum_{t=1}^{\tau}\mathbb{I}\{\sAbs_t=\overline{s} \}}{\tau}-\overline{\lambda}(\overline{s}) \right| \leq \sqrt{\frac{2}{\tau} \log \left(\frac{2}{\delta'} \right)}  \ \ \ \text{with probability at least $1-\delta'$}
\end{align*}
Since we want the statement above to hold uniformly for every state $\overline{s} \in \SspaceAbs$ and for every possible abstract policy, we set $\delta=\frac{\delta'}{\overline{S}\overline{A}^{\overline{S}}}$ and apply a union bound, obtaining:
\begin{equation}
    \left| \frac{\sum_{t=1}^{\tau}\mathbb{I}\{\overline{s}_t=\overline{s} \}}{\tau}-\lambdaAbs(\overline{s}) \right| \leq \sqrt{\frac{2}{\tau} \log \left(\frac{2\overline{S}\overline{A}^{\overline{S}}}{\delta} \right)}  \ \ \ \text{with probability at least $1-\delta$} \label{eq: PAC_bound_density}
\end{equation}
Then, we can bound $\Delta_{k+1}$ as follows.
\begin{flalign}
    \Delta_{k+1}&= \langle \nabla \mathcal{L}(\overline{\lambda}_{k}), \Tilde{\overline{\upsilon}}_{k+1}-\overline{\upsilon}^+_{k+1}\rangle \nonumber\\
    &=-\frac{1}{2 S}\sum_{\overline{s}} \frac{E_{\overline{s}}\sqrt{2\sigma^2(\overline{s})}}{\big(\lambdaAbs_k(\sAbs) + E_{\sAbs} \eta \big)^{\frac{3}{2}}} \sum_{\aAbs} \left(\Tilde{\overline{\upsilon}}_{k+1}(\overline{s}, \aAbs)-\overline{\upsilon}^+_{k+1}(\overline{s}, \aAbs) \right) \nonumber\\
    & \labelrel={step:bound_delta_1}-\frac{1}{2 S}\sum_{\overline{s}} \frac{E_{\overline{s}}\sqrt{2\sigma^2(\overline{s})}}{\big(\lambdaAbs_k(\sAbs) + E_{\sAbs} \eta \big)^{\frac{3}{2}}}\left(\Tilde{\overline{\upsilon}}_{k+1}(\overline{s})-\overline{\upsilon}^+_{k+1}(\overline{s}) \right) \nonumber\\
    &\labelrel\leq{step:bound_delta_2}-\frac{1}{2 S}\sum_{\overline{s}} \frac{E_{\overline{s}}^{-\frac{1}{2}}\sqrt{2\sigma^2(\overline{s})}}{\eta ^{\frac{3}{2}}}\left(\Tilde{\overline{\upsilon}}_{k+1}(\overline{s})-\overline{\upsilon}^+_{k+1}(\overline{s}) \right) \nonumber\\
    &\labelrel \leq {step:bound_delta_3}\frac{\Phi^{1/2}\overline{S} \sqrt{\sigma^2_{\max}}}{ S\eta^{3/2}} \left| \left| \Tilde{\overline{\upsilon}}_{k+1}-\overline{\upsilon}^+_{k+1}\right| \right |_{\infty}\nonumber\\
    &\labelrel \leq {step:bound_delta_4} \frac{\Phi^{3/2} \sqrt{\sigma^2_{\max}}}{\eta^{3/2}} \sqrt{\frac{2}{\tau_k} \log \left(\frac{2\overline{S}\overline{A}^{\overline{S}}}{\delta} \right)} \label{eq: bound_Delta}\ \ \ \text{with probability at least $1-\delta$}
\end{flalign}
where in \eqref{step:bound_delta_1} we consider the state densities by summing over the actions, in \eqref{step:bound_delta_2} we lower bound $\lambdaAbs(\sAbs) \geq 0$, in \eqref{step:bound_delta_3} we use that $E_{\sAbs}=\frac{1}{\Phi}$ in addition to taking the infinity norm and in \eqref{step:bound_delta_4} we use the PAC bound in Equation \ref{eq: PAC_bound_density}. 

\subsection{Upper Bound $\epsilon_{k+1}$}

Next, we define as $\nabla \widehat{\mathcal{L}}^{+}_{t_k-1}(\lambdaAbs)$ the empirical optimistic gradient it iteration $k$, defined as in Equation \ref{eq:abstract_reward}, but replacing the true variance with its empirical counterpart, and by using Lemma \ref{lemma:variance_upper_bound} we upper bound the true gradient $\nabla \mathcal{L}(\lambdaAbs)$ with a term containing the gradient estimate $\nabla \widehat{\mathcal{L}}_{t_k-1}(\lambdaAbs)$ (same as $\nabla \widehat{\mathcal{L}}^{+}_{t_k-1}(\lambdaAbs)$ but without the $\alpha$). More explicitly, we have that:
\begin{equation}
    \nabla \widehat{\mathcal{L}}_{t_k-1}^{+}(\lambdaAbs)(\overline{s},\aAbs)=\frac{- E_{\sAbs} \left[\sqrt{2 \hat{\sigma}_{t_k - 1}^2(\overline{s})} + \alpha(t_k -1, \overline{s}, \delta)\right]}{ 2 S \big(\lambdaAbs(\sAbs) + E_{\sAbs} \eta \big)^{\frac{3}{2}}} \nonumber
\end{equation}

\begin{equation}
    \nabla \widehat{\mathcal{L}}_{t_k-1}(\lambdaAbs)(\overline{s},\aAbs)=\frac{- E_{\sAbs} \sqrt{2 \hat{\sigma}_{t_k - 1}^2(\overline{s})}}{ 2 S \big(\lambdaAbs(\sAbs) + E_{\sAbs} \eta \big)^{\frac{3}{2}}} \nonumber
\end{equation}

\begin{equation}
    \nabla \mathcal{L}(\lambdaAbs)(\overline{s},\aAbs)=\frac{- E_{\sAbs} \sqrt{2 \sigma^2(\overline{s})}}{ 2 S \big(\lambdaAbs(\sAbs) + E_{\sAbs} \eta \big)^{\frac{3}{2}}} \nonumber
\end{equation}
In particular, we obtain that with probability at least $1-\delta$:
\begin{equation}
    \nabla \widehat{\mathcal{L}}_{t_k-1}^+(\lambdaAbs)(\overline{s},\aAbs)=  \nabla \widehat{\mathcal{L}}_{t_k-1}(\lambdaAbs)(\overline{s},\aAbs) -\frac{1}{2 \Phi S}\frac{\alpha(t_k-1, \overline{s}, \delta)}{\left(\lambdaAbs(\overline{s})+E_{\overline{s}}\eta \right)^{3/2}} \leq \nabla \mathcal{L}(\lambdaAbs)(\overline{s},\aAbs)\label{eq: gradient_bounds_1} \ \ \ \ \text{and}
\end{equation}
\begin{equation}
    \nabla \mathcal{L}(\lambdaAbs)(\overline{s},\aAbs) \leq \nabla \widehat{\mathcal{L}}_{t_k-1}(\lambdaAbs)(\overline{s},\aAbs) +\frac{1}{2 \Phi S}\frac{\alpha(t_k-1, \overline{s}, \delta)}{\left(\lambdaAbs(\overline{s})+E_{\overline{s}} \eta \right)^{3/2}} \label{eq: gradient_bounds_2}
\end{equation}
Hence, we get that:
\begin{align}\left\langle\nabla \mathcal{L}\left(\overline{\lambda}\right), \overline{\upsilon}_{k+1}^{+}\right\rangle & =\sum_{\overline{s}, \aAbs} \overline{\upsilon}_{k+1}^{+}(\overline{s}, \aAbs) \nabla \mathcal{L}\left(\overline{\lambda}\right)(\overline{s}, \aAbs) \nonumber\\ & \labelrel\leq{step:regret_4} \sum_{\overline{s}, \aAbs} \overline{\upsilon}_{k+1}^{+}(\overline{s}, \aAbs) \nabla \widehat{\mathcal{L}}_{t_k-1}\left(\overline{\lambda}\right)(\overline{s}, \aAbs)+\frac{1}{2 \Phi S}\sum_{\overline{s}, \aAbs} \overline{\upsilon}_{k+1}^{+}(\overline{s}, \aAbs) \frac{\alpha\left(t_k-1, \sAbs, \delta\right)}{\left(\lambdaAbs(\overline{s})+E_{\overline{s}}\eta \right)^{3/2}} \nonumber\\ & \labelrel\leq{step:regret_4_b}\sum_{\overline{s}, \aAbs} \overline{\upsilon}_{k+1}^{+}(\overline{s}, \aAbs) \nabla \widehat{\mathcal{L}}_{t_k-1}^+\left(\overline{\lambda}\right)(\overline{s}, \aAbs)+\frac{1}{\Phi S}\sum_{\overline{s}, \aAbs} \overline{\upsilon}_{k+1}^{+}(\overline{s}, \aAbs) \frac{\alpha\left(t_k-1, \sAbs, \delta\right)}{\left(\lambdaAbs(\overline{s})+E_{\overline{s}}\eta\right)^{3/2}} \nonumber \\ &=\left\langle\nabla \widehat{\mathcal{L}}_{t_k-1}^{+}\left(\overline{\lambda}\right), \overline{\upsilon}_{k+1}^{+}\right\rangle+\frac{1}{\Phi S} \sum_{\overline{s}, \aAbs} \overline{\upsilon}_{k+1}^{+}(\sAbs, \aAbs) \frac{\alpha\left(t_k-1, \sAbs, \delta\right)}{\left(\lambdaAbs(\overline{s})+E_{\overline{s}}\eta\right)^{3/2}} \nonumber \\ & \labelrel\leq{step:regret_5}\left\langle\nabla \widehat{\mathcal{L}}_{t_k-1}^{+}\left(\overline{\lambda}\right), \overline{\upsilon}_{k+1}^{\star}\right\rangle+ \frac{1}{\Phi S}\sum_{\overline{s}, \aAbs} \overline{\upsilon}_{k+1}^{+}(\sAbs, \aAbs) \frac{\alpha\left(t_k-1, \sAbs, \delta\right)}{\left(\lambdaAbs(\overline{s})+E_{\overline{s}}\eta\right)^{3/2}} \nonumber\\ & \labelrel\leq{step:regret_6}\left\langle\nabla \mathcal{L}\left(\overline{\lambda}\right), \overline{\upsilon}_{k+1}^{\star}\right\rangle+ \frac{3}{2\Phi S}\sum_{\overline{s}, \aAbs} \overline{\upsilon}_{k+1}^{+}(\overline{s}, \aAbs) \frac{\alpha\left(t_k-1, \overline{s}, \delta\right)}{\left(\lambdaAbs(\overline{s})+E_{\overline{s}}\eta\right)^{3/2}} \label{final_ineq_gradient_eps}\end{align}
where in \eqref{step:regret_4} we used Equation \ref{eq: gradient_bounds_2}, in \eqref{step:regret_5} we used the optimality of $\overline{\upsilon}_{k+1}^+$  and in \eqref{step:regret_4_b} and \eqref{step:regret_6} we used Equation \ref{eq: gradient_bounds_1}.

Using the definition of $\epsilon_{k+1}$ from Equation \eqref{eq: last_eq_tracking_error}, by rearranging the terms in Equation \eqref{final_ineq_gradient_eps} we get:
\begin{flalign}
    \epsilon_{k+1}&=\langle \nabla \mathcal{L}(\overline{\lambda}_k), \overline{\upsilon}^+_{k+1}-\overline{\upsilon}^*_{k+1}\rangle \nonumber \\
    &\leq \frac{3}{2\Phi S}\sum_{\overline{s}, \aAbs} \overline{\upsilon}_{k+1}^{+}(\overline{s}, \aAbs) \frac{\alpha\left(t_k-1, \overline{s}, \delta\right)}{\left(\lambdaAbs_k(\overline{s})+E_{\overline{s}}\eta\right)^{3/2}} \nonumber \\
    &\leq \frac{3}{2\Phi S}\sum_{\overline{s}, \aAbs} \overline{\upsilon}_{k+1}^{+}(\overline{s}, \aAbs) \frac{\alpha\left(t_k-1, \overline{s}, \delta\right)}{\left(E_{\overline{s}}\eta\right)^{3/2}} \label{eq: epsilon_1}
\end{flalign}
where in the last inequality we simply lower-bound $\lambdaAbs_k(\sAbs)\geq 0$.

In the following we denote with $T_k(\overline{s}) \coloneqq T_{t_k-1}(\overline{s})$ the number of visits of state $\overline{s}$ from the start until iteration $k-1$ of the FW scheme (i.e., at time $t_{k-1}$ ). We now plug in the definition of $\alpha\left(t, \overline{s}, \delta\right)=F_{\max}\sqrt{\frac{2\log(2\overline{S}t^2/\delta)}{T_{t}(\overline{s})}}$, coming from Lemma \ref{lemma:variance_upper_bound}, in Equation \ref{eq: epsilon_1}, leading to:
\begin{align}
\epsilon_{k+1} &\leq \sum_{\overline{s}, \aAbs} \overline{\upsilon}_{k+1}^{+}(\overline{s}, \aAbs) \frac{3F_{\max}}{2\Phi S{(\frac{1}{\Phi}\eta)}^{3/2}} \sqrt{2\log \left(\frac{2 \overline{S}\left(t_k-1\right)^2}{\delta}\right)} \frac{1}{\sqrt{T_k(\overline{s})}} \nonumber \\
&=\sum_{\overline{s}, \aAbs} \overline{\upsilon}_{k+1}^{+}(\overline{s}, \aAbs) \frac{3 \Phi^{1/2} F_{\max}}{2S\eta^{3/2}} \sqrt{2\log \left(\frac{2 \overline{S}\left(t_k-1\right)^2}{\delta}\right)} \frac{1}{\sqrt{T_k(\overline{s})}} \nonumber \\
&\leq c_0\sum_{\overline{s}, \aAbs} \overline{\upsilon}^+_{k+1}(\overline{s}, \aAbs) \frac{1}{\sqrt{T_k(\overline{s})}} \nonumber \\
&\labelrel={step:regret_7} c_0 \underbrace{\sum_{\overline{s}, \aAbs} \widetilde{\overline{\upsilon}}_{k+1}(\overline{s}, \aAbs) \frac{1}{\sqrt{T_k(\overline{s})}}}_{\gamma_k}+\underbrace{c_0 \sum_{\overline{s}, a}\left(\overline{\upsilon}_{k+1}^{+}(\overline{s}, \aAbs)-\widetilde{\overline{\upsilon}}_{k+1}(\overline{s}, \aAbs)\right) \frac{1}{\sqrt{T_k(\overline{s})}}}_{\Gamma_{k+1}} \label{eq: epsilon_2}
\end{align}

with $c_0=\frac{3 \Phi^{1/2} F_{\max}}{2S \eta^{3/2}} \sqrt{2\log \left(2\frac{\overline{S} T^2}{\delta}\right)}$, where $T\coloneqq t_K-1$ and $K$ is the total number of FW iterations and where in \eqref{step:regret_7} we simply added and substracted a term $ c_0\sum_{\overline{s}, \aAbs} \widetilde{\overline{\upsilon}}^+_{k+1}(\overline{s}, \aAbs) \frac{1}{\sqrt{T_k(\overline{s})}}$. We can now bound $\Gamma_{k+1}$ similarly to $\Delta_{k+1}$ using Lemma \ref{PAC guarantee}. In particular, by upper-bounding $\frac{1}{\sqrt{T_k(\overline{s})}}\leq 1$, we get that with probability at least $1-\delta$, we have that:
$$
\Gamma_{k+1} \leq c_0 \overline{S} \sqrt{\frac{2}{\tau_k} \log \left(\frac{2\overline{S}\overline{A}^{\overline{S}}}{\delta} \right)}
$$
Finally, by plugging in the bound of $\Gamma_{k+1}$ we just derived into Equation \ref{eq: epsilon_2}, we have that with probability at least $1-\delta$:
\begin{equation}
    \epsilon_{k+1} \leq c_0 \gamma_k+c_0 \overline{S} \sqrt{\frac{2}{\tau_k} \log \left(\frac{2\overline{S}\overline{A}^{\overline{S}}}{\delta} \right)} \ \ \ \text{where} \ \ \ c_0=\frac{3 \Phi^{1/2} F_{\max}}{2S\eta^{3/2}}\sqrt{2\log \left(\frac{2\overline{S}T^2}{\delta}\right)} \label{eq: bound_epsilon_3} \ \ \text{and} \ \ \ \gamma_k=\sum_{\overline{s}, \overline{a}}\frac{\Tilde{\overline{\upsilon}}_{k+1}(\overline{s}, \overline{a})}{\sqrt{T_k(\overline{s})}}
\end{equation}
In the next steps, we will introduce an explicit time dependency on $\tau_k$, which we will use to simplify the expression of the approximation error and finally perform a recursion leading to the final expression of the regret.\\
First, we recall from Equation \ref{eq: last_eq_tracking_error} that with probability at least $1-\delta$, we have:

\begin{align*}
    \rho_{k+1} \leq (1-\beta_k)\rho_k+C_{\eta}\beta_k^2+\beta_k \epsilon_{k+1}+\beta_k\Delta_{k+1}
\end{align*}

By plugging in the upper bounds of $\epsilon_{k+1}$ and $\Delta_{k+1}$ that we derived in \eqref{eq: bound_Delta} and \eqref{eq: bound_epsilon_3}, we get:

\begin{flalign}
    \rho_{k+1}&\leq (1-\beta_k)\rho_k+C_{\eta}\beta_k^2+\beta_k \left(c_0 \gamma_k+c_0 \overline{S} \sqrt{\frac{2}{\tau_k} \log \left(\frac{2\overline{S}\overline{A}^{\overline{S}}}{\delta} \right)} \right)+\beta_k \frac{\Phi^{3/2} \sqrt{\sigma^2_{\max}}}{\eta^{3/2}} \sqrt{\frac{2}{\tau_k} \log \left(\frac{2\overline{S}\overline{A}^{\overline{S}}}{\delta} \right)} \nonumber \\
    &= (1-\beta_k)\rho_k+C_{\eta}\beta_k^2+\beta_k \underbrace{\left(c_0 \overline{S}+\frac{\Phi^{3/2} \sqrt{\sigma^2_{\max}}}{\eta^{3/2}} \right) \sqrt{\frac{2}{\tau_k} \log \left(\frac{2\overline{S}\overline{A}^{\overline{S}}}{\delta} \right)}}_{c_1/\sqrt{\tau_k}}+\beta_kc_0\gamma_k \label{eq: error_with_c_1}
\end{flalign}

Choosing $t_k=\tau_1(k-1)^3+1$, we get:

\begin{equation}
    \tau_k=t_{k+1}-t_k=\tau_1(3k^2-3k+1)\geq 3\tau_1k^2 \ \ \text{and} \ \ \beta_k=\frac{\tau_k}{t_k-1} \leq \frac{3}{k} \label{eq: bound_beta_k}
\end{equation}

Hence, using \eqref{eq: bound_beta_k} and the fact that $\tau_1 \geq 1$, we can further upper-bound together the second and third terms in Equation \eqref{eq: error_with_c_1} as follows:

\begin{flalign}
    C_{\eta}\beta_k^2+\beta_k \frac{c_1}{\sqrt{\tau_k}}\leq 9\frac{C_{\eta}}{k^2}+\frac{\sqrt{3}c_1}{k^2}=:\frac{b_\delta}{k^2} \label{eq: def_b_delta}
\end{flalign}

Plugging this in \eqref{eq: error_with_c_1} gives:

\begin{align}
    \rho_{k+1}\leq (1-\beta_k)\rho_k+\frac{b_\delta}{k^2}+\beta_kc_0\gamma_k \label{ineq approx error} 
\end{align}

We now want to compute the recursion on $\rho$ in order to find the approximation error after $K$ iterations. In order to do so, we will closely follow the steps from \cite{tarbouriech2019active}, which we will report for completeness. We choose $q \geq\left(\overline{S} / \tau_1\right)^\frac{1}{3}+1$ for later use, and we introduce the sequence $\left(u_k\right)_{k \geq q}$ with $u_q=\rho_q$ and
\begin{equation}
u_{k+1}=\left(1-\frac{1}{k}\right) u_k+\frac{b_\delta}{k^2}+\beta_k c_0 \gamma_k \label{eq: def_u_k}
\end{equation}
From Inequality \ref{ineq approx error}, we have $\rho_k \leq u_k$ and by induction we can see that $\left(u_k\right) \geq 0$. By rearranging the terms in \eqref{eq: def_u_k}, we get:

\begin{align*}
(k+1) u_{k+1}-k u_k=\frac{-u_k}{k}+\frac{b_\delta(k+1)}{k^2}+(k+1) \beta_k c_0 \gamma_k \leq \frac{b_\delta(k+1)}{k^2}+(k+1) \beta_k c_0 \gamma_k
\end{align*}

Let $K \geq q$. We can now apply a telescoping sum starting at $q$ and ending at $K$ and exploit the fact that $\beta_k \leq 3 / k \leq 6 /(k+1)$ which leads to:
\begin{equation}
K u_K-q u_q \leq 2 b_\delta \sum_{k=q}^{K-1} \frac{1}{k}+6 c_0 \sum_{k=q}^{K-1} \gamma_k \leq 2 b_\delta \log \left(\frac{K-1}{q-1}\right)+6 c_0 \sum_{k=q}^{K-1} \gamma_k \label{eq: telescoping_sum}
\end{equation}

where the $qu_q$ term appears as it corresponds to the first term of the telescoping sum.

By rearranging the terms in \eqref{eq: telescoping_sum} and using that $\rho_K \leq u_K$ as previously observed, we have with probability at least $1- \delta$:

\begin{flalign}
\rho_K \leq u_K &\leq \frac{q \rho_q+2 b_\delta \log K}{K}+\frac{6 c_0}{K} \sum_{k=q}^{K-1} \gamma_k \label{eq: final_bound_error_time_K}\\
&=\frac{\tau_1^{1 / 3}}{\left(t_K-1\right)^{1 / 3}+\tau_1^{1 / 3}}\left(q \rho_q+2 b_\delta \log K+6 c_0 \sum_{k=q}^{K-1} \gamma_k\right)
\end{flalign}

By employing \citep[Lemma 6]{tarbouriech2019active} with $q \geq\left(\overline{S} / \tau_1\right)^{1 / 3}+1$ as previously set, we get the following upper bound:

\begin{align*}
    \rho_K\leq \frac{\tau_1^{1/3}}{(t_K-1)^{1/3}+\tau_1^{1/3}} \left(q \rho_q+2b_\delta \log K+6c_0 \sqrt{\frac{\Sigma}{\tau_1}} \right)
\end{align*}

where $C_{\eta}$ is the smoothness constant of the objective function, which can be bounded as in Lemma \ref{smoothness constant bound} and $\Sigma \coloneqq \overline{S} \log (\frac{\tau_1(K-1)^3}{\overline{S}})$.

To showcase better interpretability of the final result and in particular to show the advantage of exploiting symmetries as in our method, we now upper-bound the error $\rho_K$ in such a way to make its dependence on the compression coefficient explicit.\\
We choose the tightest possible $q$ and by using the bounds for $q \rho_q$ and $b_{\delta}$ from Lemma \ref{lemma: O_tilde_q_rho_q} and Lemma \ref{lemma: O_tilde_b_delta} we get:
\begin{align*}
    \rho_K &=\widetilde{\mathcal{O}} \Bigg( \frac{\tau_1^{1/3}}{(t_K-1)^{1/3}+\tau_1^{1/3}} \Bigg[\frac{A\Phi^{\frac{1}{2}}F_{\max}\sqrt{\sigma^2_{\max}} \sqrt{S}}{\eta^{\frac{5}{2}}} +\frac{A\Phi^2 F_{\max} \sqrt{\sigma_{\max }^2} \sqrt{S}}{\eta^{\frac{5}{2}}} \\
    & \:\:\:\:\:\:\:\:\:\:\: \:\:\:\:\:\:\:\:\:\:\: \:\:\:\:\:\:\:\:\: \:\:\:\:\:\:\:\:\:\:\: \:\:\:\:\:\:\:\:\:\:\:\:\:\:\:\:\:\:\: \:\:\:\:\:\:\:\:\:\:\: \:\:\:\:\:\:\:\:\: \:\:\:\:\:\:\:\:\:\:\: \:\:\:\:\:\:\:\: +\frac{3 \Phi^{1/2} F_{\max}}{2S\eta^{\frac{3}{2}}}\sqrt{2\log \left(\frac{2\overline{S}T^2}{\delta}\right)}\sqrt{\frac{\overline{S} \log (\frac{\tau_1(K-1)^3}{\overline{S}})}{\tau_1}} \Bigg] \Bigg)  \\
    &= \widetilde{\mathcal{O}}\left( \left(\frac{A\Phi^{\frac{1}{2}} S^{\frac{1}{2}}F_{\max} \sqrt{\sigma^2_{\max}}}{\eta^{\frac{5}{2}}}\right)\frac{1}{{t_K}^{\frac{1}{3}}} \right)\\&=\widetilde{\mathcal{O}}\left(\left(\frac{A\Phi^{\frac{1}{2}} S^{\frac{1}{2}}F_{\max} \sqrt{\sigma^2_{\max}}}{\eta^{\frac{5}{2}}}\right)\frac{1}{n^{\frac{1}{3}}} \right) 
\end{align*}

where in the last step we used that $n= \sum_{k=1}^{K} \tau_k = \sum_{k=1}^K \Theta(k^2)=\Theta(K^3)=\Theta(t_K)$ and
where $T=t_K-1$.\\
\end{proof}

\EstSampleComplexity*
\begin{proof}
     The result simply follows by inverting the regret guarantee in Theorem~\ref{theorem:regret_theorem}.
\end{proof}

\begin{restatable}
    {lemma}{upper bound of rho_q}\label{bound rho_q}
    We have the following bound for the approximation error at time $q$:
    \begin{align*}
        \rho_q \leq \frac{\sqrt{\Phi} \sqrt{2\sigma_{\max}^2}}{\sqrt{\eta} \cdot q}+\frac{2 b_\delta \log q}{q}+\frac{9 \cdot \Phi^{1 / 2} F_{\max}}{S\eta^{3 / 2}} \sqrt{2\log \left(\frac{2\bar{S} q^6}{\delta}\right)}
    \end{align*}
\end{restatable}

\begin{proof} 
We begin by noting that we can derive an equivalent bound of the error at time $q$ in the same way as in Equation \ref{eq: final_bound_error_time_K} by deploying the telescoping series from $1$ to $q-1$ hence getting:
\begin{flalign} 
    \rho_q &\leq \frac{\rho_1+2 b_\delta \log q}{q}+\frac{6 c_0}{q} \sum_{k=1}^{q-1} \gamma_k \nonumber\\
    &=\frac{\rho_1+2 b_\delta \log q}{q}+\frac{6 c_0}{q} \sum_{k=1}^{q-1}\sum_{\overline{s}, \overline{a}}\frac{\Tilde{\overline{\upsilon}}_{k+1}(\overline{s}, \overline{a})}{\sqrt{T_k(\overline{s})}} \nonumber\\
    &\leq \frac{\rho_1+2 b_\delta \log q}{q}+\frac{6 c_0}{q} \sum_{k=1}^{q-1}\underbrace{\sum_{\overline{s}, \overline{a}}\Tilde{\overline{\upsilon}}_{k+1}(\overline{s}, \overline{a})}_{=1} \nonumber\\
    &\leq \frac{\rho_1+2 b_\delta \log q}{q}+6 c_0 \label{eq: partial_bound_rho_q}
\end{flalign}

We now proceed in bounding the error $\rho_1$, by first recalling the definition we gave in Equation \ref{eq: definition_regret}:
\begin{align}
    \rho_1& \coloneqq L\left(\bar{\lambda}_2\right)-\underbrace{L\left(\lambda^*\right)}_{\geq 0} \nonumber\\
    &\leq L\left(\bar{\lambda}_2\right) \nonumber \\ & \leq \frac{\sqrt{\Phi} \sqrt{2\sigma_{\max}^2}}{\sqrt{ \eta}} \nonumber
    \end{align}
where in the last inequality we use Lemma \ref{bound L_infty}.\\
Hence, by plugging this bound in Equation \ref{eq: partial_bound_rho_q} and expliciting $c_0$ as given in Equation \ref{eq: bound_epsilon_3}, we get:
\begin{align}
    \rho_q &\leq \frac{\sqrt{\Phi} \sqrt{2\sigma_{\max}^2}}{\sqrt{\eta} \cdot q}+\frac{2 b_\delta \log q}{q}+6 c_0 \nonumber\\ & =\frac{\sqrt{\Phi} \sqrt{2\sigma_{\max}^2}}{\sqrt{\eta} \cdot q}+\frac{2 b_\delta \log q}{q}+\frac{9 \cdot \Phi^{1 / 2} F_{\max}}{S\eta^{3 / 2}} \sqrt{2\log \left(\frac{2\bar{S} q^6}{\delta}\right)} \nonumber\end{align}
\end{proof}

\begin{restatable}{lemma}{O_tilde of b_delta}\label{lemma: O_tilde_b_delta}
It holds that:
\begin{flalign}
    b_{\delta}=\widetilde{\mathcal{O}}\left(\frac{A\Phi^2 F_{\max} \sqrt{\sigma_{\max }^2} \sqrt{S}}{\eta^{5 / 2}}\right) \nonumber
\end{flalign}
    
\end{restatable}

\begin{proof}
Using the definition of $b_{\delta}$ from Equation \ref{eq: def_b_delta}, we have:
\begin{flalign}
b_{\delta}& =\sqrt{3}c_1+9C_{\eta} \\
&\labelrel\leq{bound_b_delta_1} \sqrt{3}\left[\left(c_0 \bar{S}+\frac{\Phi^{3 / 2} \sqrt{\sigma_{\max }^2}}{ \eta^{3 / 2}}\right) \sqrt{2 \log \left(\frac{2 \bar{S} \overline{A}^{\bar{S}}}{\delta}\right)}\right]+9 \frac{A\sqrt{2 \sigma_{\max}^2} \cdot \Phi^{5 / 2}}{ \eta^{5 / 2}} \\
& \labelrel={bound_b_delta_2}\sqrt{3}\left(\left(\frac{3 \Phi^{3 / 2} F_{\max}}{2\eta^{3 / 2}} \sqrt{2\log \left(\frac{\bar{S} T}{\delta}\right)}+\frac{\Phi^{3 / 2} \sqrt{\sigma_{\max}^2}}{\eta^{3 / 2}}\right) \sqrt{2 \log \left(2\frac{2 \bar{S} \overline{A}^{\bar{S}}}{\delta}\right)}\right)+9 \frac{A\sqrt{2 \sigma_{\max}^2} \cdot \Phi^{5 / 2}}{ \eta^{5 / 2}} 
\end{flalign}

where in \ref{bound_b_delta_1} we plug in the value of $c_1$ from Equation \ref{eq: error_with_c_1} and we use the upper bound of the smoothness constant from Lemma \ref{smoothness constant bound} and in \ref{bound_b_delta_2} we plug in the definition of $c_0$ from Equation \ref{eq: partial_bound_rho_q}. Hence, we have:

\begin{flalign}
b_{\delta}&=\widetilde{\mathcal{O}}\left( \left(\frac{3 \Phi^{3 / 2} F_{\max}}{2\eta^{3 / 2}} \sqrt{2\log \left(\frac{2\bar{S} T}{\delta}\right)}+\frac{\Phi^{3 / 2} \sqrt{\sigma_{\max}^2}}{ \eta^{3 / 2}}\right) \sqrt{\log \left(\frac{2 \bar{S}}{\delta}\right)+\Phi S \log (\overline{A})}+\frac{A\sqrt{2 \sigma_{\max}^2} \cdot \Phi^{5 / 2}}{ \eta^{5 / 2}}\right) \\
& =\widetilde{\mathcal{O}}\left(\frac{A\Phi^2 F_{\max} \sqrt{\sigma_{\max }^2} \sqrt{S}}{\eta^{5 / 2}}\right) 
\end{flalign}
\end{proof}

\begin{restatable}{lemma}{O_tilde of q_rho_q}\label{lemma: O_tilde_q_rho_q}
It holds that:
\begin{flalign*}
    q \rho_q = \widetilde{\mathcal{O}} \Bigg ( \frac{A\Phi^{\frac{1}{2}}\sqrt{\sigma^2_{\max}}F_{\max}\sqrt{S}}{\eta^{\frac{5}{2}}} \Bigg)
\end{flalign*}

\begin{proof}
From Lemma \ref{bound rho_q} we get:
    \begin{flalign} 
    q \rho_q &=\widetilde{\mathcal{O}}\Bigg( \frac{\sqrt{\Phi} \sqrt{2\sigma_{\max}^2}}{\sqrt{\eta}}+2 b_\delta \log q+q \frac{9\cdot \Phi^{1 / 2} F_{\max}}{S\eta^{3 / 2}} \sqrt{2\log \left(\frac{2 \bar{S} q^6}{\delta}\right)} \Bigg) \nonumber \\ & =\widetilde{\mathcal{O}}\Bigg (\frac{\sqrt{\Phi} \sqrt{2\sigma_{\max}^2}}{\sqrt{\eta}}+\frac{1}{3} 2 b_\delta \log \bar{S}+\bar{S}^{1 / 3} \frac{9\cdot \Phi^{1 / 2} F_{\max}}{S\eta^{3 / 2}} \Bigg)\nonumber
    \end{flalign}

where in the second equality we used the fact that $q \geq (\overline{S}/ \tau_1)^{1/3}+1$ and $\frac{1}{\tau_1}\leq 1$. Next, we plug in the result of Lemma \ref{lemma: O_tilde_b_delta} into $b_{\delta}$ in the last equation to get:
    
    \begin{flalign}
    q \rho_q&=\widetilde{\mathcal{O}}\left(\frac{\sqrt{\Phi} \sqrt{2\sigma_{\max}^2}}{\sqrt{\eta}}+\frac{A\Phi^2 F_{\max} \sqrt{\sigma_{\max}^2} \sqrt{S}}{\eta^{5 / 2}} \log \bar{S}+\Phi^{\frac{1}{3}} S^{\frac{1}{3}} \frac{9\cdot \Phi^{1 / 2} F_{\max}}{S\eta^{3 / 2}}\right) \nonumber \\
    &=\widetilde{\mathcal{O}} \Bigg ( \frac{A\Phi^{\frac{1}{2}}\sqrt{\sigma^2_{\max}}F_{\max} \sqrt{S}}{\eta^{\frac{5}{2}}} \Bigg) \nonumber
    \end{flalign}
\end{proof}
    
\end{restatable}

\begin{restatable}{lemma}{upper bound of L_n}
\label{lemma:bound_L_n}
 For $E_{\overline{s}}\eta \leq \frac{1}{n}$ and $\forall \lambda \in \Lambda$, we can bound the non asymptotic objective function as:
 \begin{align*}
     \overline{\mathcal{L}}_n\left(\lambda\right) \leqslant \overline{\mathcal{L}}_{\infty, \eta}\left(\lambda\right)+\frac{\overline{S} F_{\max}}{S \sqrt{n} \eta}
 \end{align*}
 \end{restatable}

 \begin{proof}
     \begin{align*}
        \Labs_n\left(\lambda\right)&=\frac{1}{S}\sum_{\overline{s} \in \overline{S}} \overline{\mathcal{F}}(\overline{s}, \lambda)\\
        &=\frac{1}{S} \sum_{\overline{s} \in \SspaceAbs} E_{\sAbs}\left(\sqrt{\frac{2 \sigma^2(\bar{s})}{\sum_{b \in A} \sum_{s \in[\bar{s}]}\lambda(s, b)+\frac{1}{n}}}+\frac{1}{\sqrt{n}} \frac{F_{\max}}{\sum_{b \in A} \sum_{s \in[\bar{s}]}\lambda(s, b)+\frac{1}{n}}\right) &\\
        & \labelrel\leq{step:lemma_D_1_1} \frac{1}{S} \sum_{\overline{s} \in \SspaceAbs} E_{\sAbs}\left(\sqrt{\frac{2 \sigma^2(\bar{s})}{\sum_{b \in A} \sum_{s \in[\bar{s}]}\lambda(s, b)+E_{\overline{s}}\eta}}+\frac{1}{\sqrt{n}} \frac{F_{\max}}{\sum_{b \in A} \sum_{s \in[\bar{s}]}\lambda(s, b)+E_{\overline{s}}\eta}\right) \\ & =\Labs_{\infty, \eta}(\lambda)+\frac{1}{S} \sum_{\overline{s} \in \SspaceAbs} E_{\sAbs}\frac{1}{\sqrt{n}} \frac{F_{\max}}{\sum_{b \in A} \sum_{s \in[\bar{s}]}\lambda(s, b)+E_{\overline{s}}\eta}\\&\leq\Labs_{\infty, \eta}(\lambda)+\frac{1}{S}\sum_{s \in[\bar{s}]}\frac{E_{\overline{s}} F_{\max}}{{\sqrt{n} E_{\overline{s}}\eta}}=\overline{\mathcal{L}}_{\infty, \eta}(\lambda)+\frac{\overline{S} F_{\max}}{S \sqrt{n} \eta}\\ & 
    \end{align*} 

    where in \eqref{step:lemma_D_1_1} we used that $E_{\overline{s}}\eta \leq \frac{1}{n}$ and in the last inequality we lower bounded $\lambda(s,b) \geq 0$.
 \end{proof}

\section{Auxiliary Lemmas}
\begin{restatable}[\cite{mutny2023active}, Lemma 5]
    {lemma}{PAC guarantee}\label{PAC guarantee} Let $\{\eta_t\}_{t=1}$ be an adapted sequence of probability distributions on
     $\mathcal{X} , \mathcal{P(X)}$ with respect to filtration $\mathcal{F}_{t-1}$. Likewise let $\{f_t\}_{t=1}$ be an adapted sequence of linear functionals $f_t : \mathcal{P(X)} \rightarrow \mathbb{R}$ s.t. $||f_t||_{\infty} \leq B_t$. Also, let $x_t \sim \eta_t$, and $\delta_t(x) = \mathbb{I}_{\{x_t=x\}}$, then:

     \begin{align*}
         \mathbb{P} \left( \left| \sum_{t=1}^T f_t(\delta_t-\eta_t)\right| \geq \sqrt{2 \sum_{t=1}^TB_t^2 \log \left(\frac{2}{\delta} \right)}  \right) \leq \delta
     \end{align*}
\end{restatable}

\begin{restatable}[Bound of Smoothness constant]
    {lemma}{smoothness} \label{smoothness constant bound} The smoothness constant $C_{\eta}$ can be bounded by:
    \begin{align*}
        C_\eta \leq \frac{A\sqrt{2 \Phi^5} \sqrt{\sigma_{\max}^2}}{\eta^{5/2}}
    \end{align*}
\end{restatable}

\begin{proof}
    Given the objective $\overline{\mathcal{L}}(\lambdaAbs)$, we have that its Hessian is made up of second order partial derivatives of the form
    \begin{equation}
        \frac{\partial^2 \overline{\mathcal{L}}(\lambdaAbs)}{\partial \lambdaAbs(\overline{s}',\overline{a}')^2} = \frac{3}{4}\frac{1}{\Phi S} \sqrt{\frac{2 \sigma^2(\overline{s}')}{\left(\lambdaAbs(\overline{s}')+E_{\overline{s}}\eta\right)^5}}
    \end{equation}
    when both partial derivatives are taken \wrt the same coordinate, while the mixed second order partial derivatives are given by:
    \begin{equation}
        \frac{\partial \overline{\mathcal{L}}(\lambdaAbs)}{\partial \lambdaAbs(s',a') \partial \lambdaAbs(s'',a'')} = 0
    \end{equation}
    Hence the the Hessian $H(\overline{\mathcal{L}})$ is a diagonal matrix, thus containing only its eigenvalues. In particular, we can upper bound the value of the biggest eigenvalue as:
    \begin{equation}
        \max_{v \in \sigma(H(\overline{\mathcal{L}}))} v \leq \sum_{(\sAbs,\aAbs) \in \SspaceAbs \times \AspaceAbs} H(\overline{\mathcal{L}})(\lambdaAbs)_{\left((\sAbs, \aAbs), (\sAbs, \aAbs)\right)}
    \end{equation}
    which corresponds to summing all entries on the diagonal of the Hessian, and where $\sigma(A)$ stands for the spectrum of $A$. Hence, noting that $E_{\overline{s}}=\frac{1}{\Phi}$ we have that $\forall \; \lambdaAbs \in \overline{\Lambda}$:
    \begin{align}
        \max_{v \in \sigma(H(\overline{\mathcal{L}}))} v &\leq \sum_{(\sAbs,\aAbs) \in \SspaceAbs \times \AspaceAbs} H(\overline{\mathcal{L}})(\lambdaAbs)_{\left((\sAbs, \aAbs), (\sAbs, \aAbs)\right)} \\
        &= \sum_{(\sAbs,\aAbs) \in \SspaceAbs \times \AspaceAbs} \frac{\partial^2 \overline{\mathcal{L}}(\lambdaAbs)}{\partial \lambdaAbs(\sAbs\,\aAbs)^2} \\
        &= \sum_{(\sAbs,\aAbs) \in \SspaceAbs \times \AspaceAbs}\frac{3}{4}\frac{1}{\Phi S} \sqrt{\frac{2 \sigma^2(\overline{s})}{\left(\frac{1}{\Phi}\eta+\lambdaAbs(\overline{s})\right)^5}}\\
        &\leq \sum_{(\sAbs,\aAbs) \in \SspaceAbs \times \AspaceAbs} \frac{1}{S}\frac{\sqrt{2 \Phi^3 \sigma^2(\sAbs)}}{\eta^{5/2}}
    \end{align}

    where in the last step we have used the fact that $\lambdaAbs \in \overline{\Lambda}$.\newline
    Since $\overline{\mathcal{L}}$ is twice continuously differentiable over $\overline{\Lambda}$, by \citep[Theorem 2.1.6]{10.5555/2670022}
    , we have that:
    \begin{equation}
        H(\overline{\mathcal{L}})\left(\lambdaAbs \right) \preceq \sum_{(\sAbs,\aAbs) \in \SspaceAbs \times \AspaceAbs} \frac{1}{S}\frac{\sqrt{2 \Phi^3 \sigma^2(\sAbs)}}{\eta^{5/2}}\mathbb{I}
    \end{equation}
    where $\mathbb{I}$ is the identity matrix, implying that for $C_\eta \leq \sum_{(\sAbs,\aAbs) \in \SspaceAbs \times \AspaceAbs} \frac{1}{S}\frac{\sqrt{2 \Phi^3 \sigma^2(\sAbs)}}{\eta^{5/2}}$, $\overline{\mathcal{L}}$ is $C_\eta$-smooth on $\overline{\Lambda}$.

    From this, we have that:
    \begin{align*}
        C_\eta \leq \frac{\overline{A}\sqrt{2 \Phi^5} \sqrt{\sigma_{\max}^2}}{\eta^{5/2}} \leq \frac{A\sqrt{2 \Phi^5} \sqrt{\sigma_{\max}^2}}{\eta^{5/2}}
    \end{align*}

    where in the last step we used the fact that $\overline{A}\leq A$, since we want to give a bound depending on the quantities of the original MDP.
\end{proof}

\begin{restatable}
    {lemma}{upper-bound of L infty} \label{bound L_infty}
    The following bound holds for the asymptotic objective:
    \begin{align*}
        \overline{\mathcal{L}}_{\infty, \eta}(\lambdaAbs)\leq \frac{\sqrt{\Phi}\sqrt{2\sigma^2_{\max}}}{\sqrt{\eta}}
    \end{align*}
\end{restatable}

\begin{proof}

Using that $\overline{\lambda}(\bar{s}, a):=\sum_{s \in[\bar{s}]} \lambda(s, a) $, we get:
\begin{align*}
\overline{\mathcal{L}}_{\infty, \eta}&\left(\lambdaAbs\right)=\frac{1}{\Phi S} \sum_{\overline{s} \in \overline{S}} \frac{\sqrt{2 \sigma^2(\sAbs)}}{\sqrt{\lambdaAbs(\sAbs)+E_{\sAbs}\eta}} \\
& \leq \frac{1}{\Phi S} \sqrt{2\sigma^2_{\max }} \sum_{\overline{s} \in \overline{S}} \frac{1}{\sqrt{\lambdaAbs(\sAbs)+E_{\sAbs}\eta} }\\
& =\frac{1}{\Phi S} \sqrt{2\sigma^2_{\max }} \sum_{\overline{s} \in \overline{S}} \frac{1}{\sqrt{\lambdaAbs(\sAbs)+E_{\sAbs}\eta} }\\
& \labelrel\leq{step:Lemma_E_4_1}\frac{ \sqrt{2\sigma^2_{{\max }}}}{\Phi S} \sum_{\bar{s} \in \bar{S}} \frac{1}{\sqrt{ E_{\overline{s}} \eta}} \\
& \labelrel={step:Lemma_E_4_2}\frac{\sqrt{2\sigma^2_{\max }}}{\Phi S} \frac{\bar{S}}{\sqrt{\frac{ \eta}{\Phi}}} \\
& \labelrel={step:Lemma_E_4_3}\frac{ \sqrt{2\sigma^2_{\max }}}{\sqrt{\Phi} S} \frac{\Phi S}{\sqrt{ \eta}} \\ 
& =\frac{\sqrt{\Phi} \sqrt{2\sigma^2_{\max }}}{\sqrt{ \eta}} 
\end{align*}

where in \eqref{step:Lemma_E_4_1} we used that $\lambdaAbs(\sAbs) \geq 0$, in \eqref{step:Lemma_E_4_2} we used that $ E_s=\frac{1}{\Phi}$ and in \eqref{step:Lemma_E_4_3} we used that $\overline{S}=\Phi S$.
\end{proof}

\compressionViaGroupCardinality*
\begin{proof}
    We consider the set $\Sspace$, the group $\group = (\{L_g\}_{g \in G}, \cdot) = (G, \cdot)$ of transformations acting on $\Sspace$ via the group action $\ast: G \times \Sspace \to \Sspace$. Due to Assumption \ref{assumption:eq_classes_same_cardinality} we have the following.
    \begin{align}
        \frac{1}{\Phi} &\labelrel={step:proof_group_card_1} E_s\\
        &= |[s]|\\
        &\labelrel={step:proof_group_card_2} |\mathrm{Orbit}(s)|\\
        &\labelrel={step:proof_group_card_3} \frac{|G|}{|\mathrm{Stab}(s)|}
    \end{align}
    where step \eqref{step:proof_group_card_1} is due to Assumption \ref{assumption:eq_classes_same_cardinality}, in step \eqref{step:proof_group_card_2} we employ the definition of Orbit~\citep[Chapter C-1, page 6]{rotman2010advanced}, and in step \eqref{step:proof_group_card_3} we leverage the Orbit-Stabilizer Theorem~\citep[Theorem C-1.16]{rotman2010advanced}.
\end{proof}

\newpage

\section{Experimental Details}
\label{sec:experimental_details}
In the following, we provide further details about the experiments carried out in this work. We first present the environments, their invariances and resulting abstract environments. Subsequently, we provide additional information regarding the implementation of \AlgNameShort.

\subsection{Pollutant Diffusion Process}
\textbf{MDP:} ~~ We consider the problem of actively measuring the amount of pollutant released to the environment. The pollutant is released from a point source and spreads radially outwards through a diffusion process. The measurement setup is displayed in Figure \ref{fig:diffusion_drawing}. As can be seen, the measurement stations are aligned in two ways. \textbf{1)} Radially outwards from the point source. This allows to measure the variation of the pollutant amount in the radial direction. We will refer to a set of states that are aligned in this way as a ray. \textbf{2)} Along circles of different radii to measure the variation in the azimuthal direction. We will refer to a set of states aligned in this way as a circle. Each measurement station corresponds to a state in the MDP where the agent obtains noisy measurements of the pollutant amount. In our setup, there are a total of 30 circles and 8 rays, leading to a total of 240 states. The action space consists of five actions, $\{\mathrm{in}, \mathrm{out}, \mathrm{clockwise}, \mathrm{anticlockwise}, \mathrm{stay}\}$. We consider both deterministic and stochastic  dynamics. In the deterministic case, if the agent chooses action $\mathrm{in}$ it moves one state closer to the point source along the ray (one further away for $\mathrm{out}$). The actions leading to transitions along the circle are $\mathrm{clockwise}$ and $\mathrm{anticlockwise}$. The action $\mathrm{stay}$, makes the agent to remain in the same state and repeat the experiment. In the case of stochastic transitions, the agent moves to it's intended state with probability $q$ and with probability $1-q$ to another reachable state randomly chosen. In the experiments conducted we used $q=0.98$.

\textbf{Invariances of $f$:} ~~ We consider the case where the diffusion process is radially symmetric, meaning that for all the states on a circle, the pollution is the same. Therefore, $f$ is invariant over different rays, as illustrated already in \ref{fig:diffusion_drawing}.

\textbf{Abstract MDP:} ~~ The aforementioned invariances on $f$ make it possible to define MDP homomorphisms $h$ mapping the original MDP to abstract MDPs. In the experiments conducted, we considered three different homomorphisms, resulting in three different geometric compression terms $\Phi$. The homomorphisms simply differ by how the rays are compressed. The first homomorphism $h_1$, maps two consecutive rays into one, resulting in a compression of the 8 rays into 4. The second homomorphism $h_2$, compresses the 8 rays into 2 and the third $h_3$ maps all 8 rays into 1 resulting in the compression illustrated in Figure \ref{fig:diffusion_drawing}. The state map $\psi$ therefore maps states on consecutive rays together. The state-dependent action map $\phi_s$ corresponds to the identity map.

\textbf{Implementation Details:} ~~ The function $f(s)$, is modeled as increasing the closer the state is to the source. From the first equivalence class on the innermost circle to the 30th equivalence class consisting of the outermost circle, $f$ decreases in steps of 300, starting at 9300. The noise $\nu(\overline{s})$ is modeled as increasing with function value as often large measurements are associated with larger variances. The corresponding standard deviations are hence also decreasing from the states closest to the source to the outermost states. The standard deviations decrease by steps of 100, starting at 3100. The distribution of $\nu(\overline{s})$ was taken to be a 0 mean Gaussian with the standard deviations given above. As an example, for a state $s$ on the 21st circle we have $f(s) = 3000$ and $\nu(s) \sim \mathcal{N}(0,1000)$. The smoothness parameter was chosen to be $\eta = 0.001$, and $\delta = 0.01$ for both, deterministic and stochastic dynamics. Furthermore, we found that in practice, a constant number of interactions $\tau_k = \tau$ for all the $K$ iterations of \AlgNameShort works well, especially for remarkably low $\tau$. In this setting, we chose $\tau = 3$, which makes the algorithm more adaptive, resulting in the rapid exploration of different equivalence classes.To update the abstract state-action frequency $\overline{\lambda}_{k+1}$, we also use a constant update step of $0.005/\overline{S}$. The initial state of the agent was chosen on the outermost circle. $n$ as 210, resulting in $K=70$ iterations of \AlgNameShort. All the experiments were repeated over 15 random seeds. The computational time was measured using a standard time library in Python. The main part of the computational time can be attributed to solving the MDP using value iteration. We applied the Bellman optimality operator until there was no change in the value function up to the 5th digit.

\subsection{Toxicity of Chemical Compounds}
\textbf{MDP:} ~~ As a second experiment we consider the experimental design problem of estimating the toxicity of chemical compounds. Similarly as in \cite{schreck2019learning, dong2022deep, Thiede_2022}, we consider an MDP where a  chemical compound is represented as a string where every character in the string stands for a base chemical element. The goal of the agent is to estimate the toxicity associated to all possible compounds that can be generated using the base chemical elements. In our setting we consider three base chemical elements \texttt{A}, \texttt{B}, and \texttt{C}. We limit the maximum length of the string to 5. The state space therefore consists of all possible chemical compounds that have at most 5 base elements and the cardinality is $S = 363$. The action space consists of the three base elements, and an action that makes the agent stay in the same state $\mathcal{A} = \{\texttt{A}, \texttt{B}, \texttt{C}, \mathrm{stay}\}$. By taking an action corresponding to a base element, the agent appends this element to the current compound. Once the agent reaches a compound with 5 base elements it can either choose to measure the toxicity of that compound again by picking the action $\mathrm{stay}$ or pick a new base element to start another compound. The agent may therefore transition from one compound $s$ of length $l$ to another compound $s'$ of length $l + 1$ if the first $l$ base elements are the same.

\textbf{Invariances of $f$:} We assume that the toxicity of a chemical compound is invariant under permutations of the compound, such that $f(s) = f(s')$ if $s$ is a permutation of $s'$. The abstract state-space therefore has a cardinality of $\overline{S} = 55$

\textbf{Abstract MDP:} ~~ These invariances on $f$ again allow us to define an MDP homomorphism that maps the original MDP to an abstract one. In this case the state map $\psi$ maps all the states which are equivalent up to permutation to one abstract state. The state-dependent action map $\phi_s$ is simply the identity. This results in an abstract MDP where the agent can transition from one abstract state $\overline{s}$ to another one $\overline{s}'$ if all the base compounds making up $\overline{s}$ are also contained in $\overline{s}'$ and the agent chooses the action corresponding to the chemical compound that is not yet in $\overline{s}'$. As an example consider $\overline{s} = \texttt{CABA}$ and $\overline{s}' = \texttt{AABCC}$, then the agent may transition from $\overline{s}$ to $\overline{s}'$ by choosing action $\texttt{C}$

\textbf{Implementation Details:} We model the toxicity of the chemical compounds and the noise as a piecewise constant functions of it's base chemical elements. For every \texttt{A} in the compound, $f$ is increased by $200$, for every \texttt{B}, \texttt{C} there is an increase of $400$ and $600$ respectively. Similarly, to the diffusion environment, we assume that higher measurements of toxicity are associated with higher standard deviations and that the noise has a Gaussian distribution with 0 mean. For every \texttt{A} the standard deviation increases by $100$, for every \texttt{B}, \texttt{C}, the standard deviations increase by $200$ and $300$ respectively. As an example consider the compound \texttt{AABC}, then $f(\texttt{AABC}) = 1400$ and $\nu(\texttt{AABC}) \sim \mathcal{N}(0,700)$. We let \AlgNameShort run for $n =2400$ with a constant step size of $\tau_k = 20\:\: \forall k \in [K]$, resulting in a total of $K=80$ iterations of \AlgNameShort. The smoothness constant chosen was $\eta = 0.0007$ and $\delta = 0.01$. To update the abstract state-action frequency $\overline{\lambda}_{k+1}$, we also use a constant update step of $0.005/\overline{S}$. The experiments were repeated over 15 different random seeds.
\newpage

\end{appendix}
\end{document}